\documentclass[final]{siamltex}
\usepackage{amsmath}
\usepackage{amsfonts}
\usepackage{graphicx}
\usepackage{geometry}
\usepackage{setspace}
\usepackage{multirow}
\usepackage{algorithmic}
\usepackage{algorithm}
\usepackage{euler}
\numberwithin{algorithm}{section}

\newtheorem{thm}{Theorem}[section]
\newtheorem{cor}{Corollary}[section]
\newtheorem{lem}{Lemma}[section]
\newtheorem{prop}{Proposition}[section]
\newtheorem{rem}{Remark}[section]
\newcommand{\sbLetter}[1]{\boldsymbol{\mathscr{#1}}}
\newcommand{\sbX}{\sbLetter{X}}
\newcommand{\sbXbar}{\bar{\sbX}}
\newcommand{\sbXtild}{\tilde{\sbX}}
\newcommand{\sbP}{\sbLetter{P}}
\newcommand{\sbPtild}{\tilde{\sbP}}
\newcommand{\sbY}{\sbLetter{Y}}
\newcommand{\sbYtild}{\tilde{\sbY}}
\newcommand{\sbZ}{\sbLetter{Z}}

\newcommand{\sbE}{\sbLetter{E}}
\newcommand{\sbB}{\sbLetter{B}}
\newcommand{\sbBtild}{\tilde{\sbB}}
\newcommand{\sbD}{\sbLetter{D}}
\newcommand{\sbG}{\sbLetter{G}}
\newcommand{\sbU}{\sbLetter{U}}
\newcommand{\sbV}{\sbLetter{V}}
\newcommand{\sbW}{\sbLetter{W}}
\newcommand{\sbLambda}{\sbLetter{\Lambda}}

\newcommand{\bA}{\mathbf{A}}
\newcommand{\bB}{\mathbf{B}}

\newcommand{\bD}{\mathbf{D}}
\newcommand{\bE}{\mathbf{E}}

\newcommand{\bI}{\mathbf{I}}

\newcommand{\bM}{\mathbf{M}}
\newcommand{\bN}{\mathbf{N}}

\newcommand{\bU}{\mathbf{U}}
\newcommand{\bV}{\mathbf{V}}

\newcommand{\bX}{\mathbf{X}}
\newcommand{\bY}{\mathbf{Y}}
\newcommand{\bZ}{\mathbf{Z}}
\newcommand{\ba}{\mathbf{a}}

\newcommand{\bd}{\mathbf{d}}

\newcommand{\bx}{\mathbf{x}}

\newcommand{\bzero}{\mathbf{0}}

\newcommand{\bLambda}{\boldsymbol{\Lambda}}

\newcommand{\modeLetter}[2]{#1_{(#2)}}
\newcommand{\modenX}{\modeLetter{\bX}{n}}
\newcommand{\modeiX}{\modeLetter{\bX}{i}}

\newcommand{\modenn}[3]{#1_{#2,(#3)}}
\newcommand{\modeiXi}{\modenn{\bX}{i}{i}}
\newcommand{\modeiLambdai}{\modenn{\bLambda}{i}{i}}

\newcommand{\modeiZ}{\modeLetter{\bZ}{i}}

\newcommand{\supInd}[1]{^{(#1)}}

\newcommand{\supIndk}{\supInd{k}}
\newcommand{\supIndkpone}{\supInd{k+1}}
\newcommand{\supIndK}{\supInd{K}}

\newcommand{\transp}{^\top}
\newcommand{\inv}{^{-1}}

\newcommand{\calA}{\mathcal{A}}
\newcommand{\calAc}{\calA_{\Omega}}
\newcommand{\calC}{\mathcal{C}}

\newcommand{\calH}{\mathcal{H}}

\newcommand{\calL}{\mathcal{L}}
\newcommand{\calO}{\mathcal{O}}
\newcommand{\calP}{\mathcal{P}}
\newcommand{\calQ}{\mathcal{Q}}

\newcommand{\calS}{\mathcal{S}}
\newcommand{\calT}{\mathcal{T}}
\newcommand{\calU}{\mathcal{U}}

\newcommand{\calSigma}{\mathcal{\Sigma}}

\newcommand{\bbR}{\mathbb{R}}

\newcommand{\setOneToN}[1]{\left\{#1\right\}_{i=1}^N}

\newcommand{\setkOneToInf}[1]{\left\{#1\right\}_{k=1}^\infty}
\newcommand{\setjOneToInf}[1]{\left\{#1\right\}_{j=1}^\infty}

\newcommand{\rmvec}{\textrm{vec}}
\newcommand{\xTimesToX}[2]{#1\times\cdots\times #2}
\newcommand{\IoneTimesToIN}{\xTimesToX{I_1}{I_N}}
\newcommand{\NIoneTimesToIN}{\xTimesToX{N\cdot I_1}{I_N}}

\newcommand{\sumOneToN}{\sum_{i=1}^N}
\newcommand{\vbar}{\;|\;}

\newcommand{\eqnBegin}{\begin{equation}}
\newcommand{\eqnEnd}{\end{equation}}
\newcommand{\eqnSBegin}{\begin{equation*}}
\newcommand{\eqnSEnd}{\end{equation*}}
\newcommand{\eqnarrayBegin}{\begin{eqnarray}}
\newcommand{\eqnarrayEnd}{\end{eqnarray}}
\newcommand{\eqnarraySBegin}{\begin{eqnarray*}}
\newcommand{\eqnarraySEnd}{\end{eqnarray*}}

\newcommand{\fold}{\textrm{fold}}
\newcommand{\unfold}{\textrm{unfold}}

\title{Robust Low-rank Tensor Recovery: Models and Algorithms}

\author{Donald Goldfarb$^*$ \and Zhiwei (Tony) Qin\thanks{Department of Industrial Engineering and Operations Research, Columbia University, New York, NY. Email: goldfarb@columbia.edu, zq2107@columbia.edu.}}

\begin{document}

\maketitle

\begin{abstract}
Robust tensor recovery plays an instrumental role in robustifying tensor decompositions for multilinear data analysis against outliers, gross corruptions and missing values  and has a diverse array of applications.  In this paper, we study the problem of robust low-rank tensor recovery in a convex optimization framework, drawing upon recent advances in robust Principal Component Analysis and tensor completion.  We propose tailored optimization algorithms with global convergence guarantees for solving both the constrained and the Lagrangian formulations of the problem.  These algorithms are based on the highly efficient alternating direction augmented Lagrangian and accelerated proximal gradient methods.  We also propose a nonconvex model that can often improve the recovery results from the convex models.  We investigate the empirical recoverability properties of the convex and nonconvex formulations and compare the computational performance of the algorithms on simulated data.  We demonstrate through a number of real applications the practical effectiveness of this convex optimization framework for robust low-rank tensor recovery.
\end{abstract}



\pagestyle{myheadings}
\thispagestyle{plain}
\markboth{D. GOLDFARB AND Z. QIN}{ROBUST LOW-RANK TENSOR RECOVERY}

\section{Introduction}\label{sec:intro}
The rapid advance in modern computer technology has given rise to the wide presence of multidimensional data (tensor data).  Traditional matrix-based data analysis is inherently two-dimensional, which limits its usefulness in extracting information from a multidimensional perspective.  On the other hand, tensor-based multilinear data analysis has shown that tensor models are capable of taking full advantage of the multilinear structures to provide better understanding and more precision.  At the core of multilinear data analysis lies tensor decomposition, which commonly takes two forms: CANDECOMP/PARAFAC (CP) decomposition \cite{carroll1970analysis,harshman1970foundations} and Tucker decomposition \cite{tucker1966some}.  Having originated in the fields of psychometrics and chemometrics, these decompositions are now widely used in other application areas such as computer vision \cite{vasilescu2003multilinear}, web data mining \cite{franz2009triplerank,sun2005cubesvd}, and signal processing \cite{abdallah2007mpeg}.

Tensor decomposition faces three major challenges: arbitrary outliers, missing data/partial observations, and computational efficiency.
Tensor decomposition resembles Principal Component Analysis (PCA) for matrices in many ways.  In fact, Tucker decomposition is also known as \emph{higher-order SVD} (HOSVD) \cite{de2000multilinear}.  It is well-known that PCA is sensitive to outliers and gross corruptions (non-Gaussian noise).  Since the CP and Tucker decompositions are also based on least-squares approximation, they are prone to these problems as well.  Algorithms based on non-convex formulations have been proposed to robustify tensor decompositions against outliers \cite{engelen2009fully,pravdova2001robust} and missing data \cite{acar2010scalable}.  However, they suffer from the lack of global optimality guarantees.

In practice, the underlying tensor data is often low-rank, even though the actual data may not be due to outliers and arbitrary errors.  In other words, the major part of the variation in the data is often governed by a relatively small number of latent factors.  It is thus possible to robustify tensor decompositions by reconstructing the low-rank part of the noisy data.  Besides its importance to tensor decompositions, robust low-rank tensor recovery also has many applications in its own right, e.g., shadow/occulsion removal in image processing.
Motivated by the aforementioned challenges and opportunities, we study the problem of low-rank tensor recovery that is robust to gross corruptions and missing values in a convex optimization framework.  Our work in this paper is built upon two major lines of previous work: Principal Component Pursuit (PCP) for Robust PCA \cite{candes2009robust} and Tensor Completion \cite{gandy2011tensor,liu2009tensor,tomioka2010estimation}.
The main advantages of the convex formulation that we use are that it removes outliers or corrects corruptions based on the global structure of the tensor, and its solution can be obtained by efficient convergent convex optimization algorithms that are easy to implement.  Moreover, this solution naturally leads to a robust Tucker decomposition, and the CP decomposition can also be obtained by applying a simple heuristic \cite{tomioka2010estimation}.


The structure of our paper is as follows.  In Section \ref{sec:notation} and \ref{sec:tensor_decomp}, we introduce notation and review some tensor basics.  We formulate robust tensor recovery problem as a convex program in Section \ref{sec:horpca}.  Two approaches for recovering a low-Tucker-rank tensor are considered.  We propose alternating direction augmented Lagrangian (ADAL) methods and accelerated proximal gradient (APG) methods, respectively, for solving the exact constrained version and the Lagrangian version of the robust tensor recovery problem.  All of the proposed algorithms have global convergence guarantees.  In Section \ref{sec:ncx}, we introduce a non-convex formulation which can nevertheless be solved approximately by an ADAL-based algorithm and can potentially take better advantage of more precise rank information.  In Section \ref{sec:exp}, we test the computational performance of the algorithms and analyze the empirical recovery properties of the models on synthetic data, and we demonstrate the practical effectiveness of our algorithms and models through a series of applications in 3D MRI recovery, fluorescence EEM data analysis, face representation, and foreground filtering/background reconstruction of a web game frame sequence.

\subsection{Mathematical Notation and Tensor Basics}\label{sec:notation}
As in \cite{kolda2009tensor}, we denote tensors by boldface Euler script letters, e.g., $\sbX$, matrices by boldface capital letters, e.g., $\bX$, vectors by boldface lowercase letters, e.g., $\bx$, and scalars by lowercase letters, e.g., x.  The \textit{order} of a tensor is the number of dimensions (a.k.a. \textit{ways} or \textit{modes}).  Let $\sbX$ be an $N$th-order tensor.  A \textit{fiber} of $\sbX$ is a column vector defined by fixing every index of $\sbX$ but one.  So for a matrix, each row is a mode-2 fiber.  The mode-$i$ \textit{unfolding} (matricization) of the tensor $\sbX$ is the matrix denoted by $\modeiX$ that is obtained by arranging (lexicographically in the indices other than the $i$-th index) the mode-$i$ fibers as the columns of the matrix.  The vectorization of $\sbX$ is denoted by $\rmvec(\sbX)$.
\begin{description}
  \item[Inner product and norms] The inner product of two tensors $\sbX,\sbY \in \bbR^{I_1\times\cdots\times I_N}$ is defined as $\langle\sbX,\sbY\rangle := \rmvec(\sbX)^T\rmvec(\sbY)$, and the Frobenius norm of $\sbX$ is denoted by $\|\sbX\| := \sqrt{\langle \sbX,\sbX \rangle}$.  The nuclear norm (or trace norm) $\|\bX\|_*$ of a matrix $\bX$ is the sum of its singular values, i.e. $\|\bX\|_* := \sum_i \sigma_i$, where the SVD of $\bX = \bU\textrm{diag}(\sigma)\bV^T$.  The $L_1$ norm of a vector $\bx$ is defined as $\|\bx\|_1 := \sum_i |x_i|$.  Likewise, for a matrix $\bX$ and a tensor $\sbX$, $\|\bX\|_1 := \|\rmvec(\bX)\|_1$, and $\|\sbX\|_1 := \|\rmvec(\sbX)\|_1$.
  \item[Vector outer product] The vector outer product is denoted by the symbol $\circ$.  The outer product of $N$ vectors, $\ba\supInd{n} \in \bbR^{I_n}, n = 1,\cdots,N$ is an $N$-th-order tensor, defined as
      \begin{equation*}
        \left( \ba\supInd{1}\circ\ba\supInd{2}\circ\cdots\circ\ba\supInd{N}\right)_{i_1 i_2\cdots i_N} := a_{i_1}\supInd{1}a_{i_2}\supInd{2}\cdots a_{i_N}\supInd{N}, \quad \forall 1\leq i_n \leq I_n, 1 \leq n \leq N.
      \end{equation*}

  \item[Tensor-matrix multiplication] The multiplication of a tensor $\sbX$ of size $I_1 \times I_2 \times \cdots \times I_N$ with a matrix $\bA \in \bbR^{J\times I_n}$ in mode $n$ is denoted by
    $\sbX\times_n\bA = \sbY \in \bbR^{I_1\times\cdots\times I_{n-1}\times J \times I_{n+1}\times\cdots\times I_N}$,
  and is defined in terms of mode-$n$ unfolding as $\modeLetter{\bY}{n} := \bA\modenX$.


  \item[Linear operator and adjoint] We denote linear operators by capital letters in calligraphic font, e.g. $\calA$, and $\calA(\sbX)$ denotes the result of applying the linear operator $\calA$ to the tensor $\sbX$.  The adjoint of $\calA$ is denoted by $\calA^*$.

  \item[Homogeneous tensor array] For the ease of notation and visualization, we define a \textit{homogeneous tensor array} (or \textit{tensor array} for short) as the tensor obtained by stacking a set of component tensors of the same size along the first mode.  An $N$-component tensor array
      $\sbXbar := \left(
         \begin{array}{c}
           \sbX_1 \\
           \vdots \\
           \sbX_N \\
         \end{array}
       \right) \in \bbR^{N\cdot I_1\times\cdots\times I_N}
      $ is a `vector' of homogeneous tensors, which we write as TArray($\sbX_1,\cdots,\sbX_N$).  A linear operator defined on a tensor array operates at the level of the component tensors.  As an example, consider the linear (summation) operator $\calA: \bbR^{\NIoneTimesToIN}\rightarrow\bbR^{\IoneTimesToIN}$ such that $\calA(\sbXbar) := \sumOneToN\sbX_i$.  Its adjoint is then the linear operator
        $\calA^*:\bbR^{I_1\times\cdots\times I_N}\rightarrow \bbR^{N\cdot I_1\times\cdots\times I_N}$ such that $\calA^*(\sbX) := \textrm{TArray}(\sbX,\cdots,\sbX)$.
      We use the non-calligraphic $\bA$ to denote the matrix corresponding to the equivalent operation carried out by $\calA$ on the mode-1 unfolding $\modeLetter{\bar{\bX}}{1}$ of $\sbXbar$, where $\modeLetter{\bar{\bX}}{1} = \textrm{TArray}(\modenn{\bX}{1}{1},\cdots,\modenn{\bX}{N}{1})$.  In this example, therefore, $\bA = \left(
                                                       \begin{array}{ccc}
                                                         \bI & \cdots & \bI \\
                                                       \end{array}
                                                     \right) \in \bbR^{I_1 \times N\cdot I_1}
      $.  

\end{description}

\subsection{Tensor decompositions and ranks}\label{sec:tensor_decomp}

The CP decomposition approximates $\sbX$ as
    $\sbX \approx \sum_{r=1}^R \lambda_r \ba_r\supInd{1} \circ \ba_r\supInd{2} \circ \cdots \circ \ba_r\supInd{N}$,
where $R > 0$ is a given integer, $\lambda_r \in \bbR$ and $\ba_r\supInd{n} \in \bbR^{I_n}$ for $r = 1,\cdots,R$.  The CP decomposition is formulated as a non-convex optimization problem and is usually computed via an alternating least-squares (ALS) algorithm; see, e.g., \cite{tomasi2006comparison}.
The rank of $\sbX$, denoted by rank($\sbX$), is defined as the smallest value of $R$ such that the approximation holds with equality.  Computing the rank of a specific given tensor is NP-hard in general \cite{hastad1990tensor}.

The Tucker decomposition approximates $\sbX$ as
    $\sbX \approx \sbG \times_1 \bU\supInd{1} \times_2 \bU\supInd{2} \cdots \times_N \bU\supInd{N}$,
where $\sbG \in \bbR^{r_1\times\cdots\times r_N}$ is called the \emph{core tensor}, and the factor matrices $\bU\supInd{n} \in \bbR^{I_n\times r_n}, n = 1,\cdots,N$ are all column-wise orthonormal.  $(r_1\times\cdots\times r_N)$ are given integers.
The $n$-rank (or mode-$n$ rank) of $\sbX$, denoted by $\textrm{rank}_n(\sbX)$, is the column rank of $\modenX$.  The set of $N$ $n$-ranks of a tensor is also called the Tucker rank.  If $\sbX$ is of rank-($r_1,\cdots,r_N$), then the approximation holds with equality, and for $n=1,\cdots,N$, $\bU\supInd{n}$ is the matrix of the left singular vectors of $\modenX$.  The Tucker decomposition is also posed as a non-convex optimization problem.  A widely-used approach for computing the factor matrices is called the \emph{higher-order orthogonal iteration} (HOOI) \cite{de2000best}, which is essentially an ALS method based on computing the dominant left singular vectors of each $\modenX$.  

\subsection{Robust PCA}
PCA gives the optimal low-dimensional estimate of a given matrix under additive i.i.d. Gaussian noise, but it is also known to be susceptible to gross corruptions and outliers.  Robust PCA (RPCA) is a family of methods that aims to make PCA robust to large errors and outliers.  Candes et. al. \cite{candes2009robust} proposed to approach RPCA via Principal Component Pursuit (PCP), which decomposes a given observation (noisy) matrix $\bB$ into a low-rank component $\bX$ and a sparse component $\bE$ by solving the optimization problem
  $\min_{\bX,\bE} \big\{ \textrm{rank}(\bX) + \lambda\|\bE\|_0 \;|\; \bX + \bE = \bB\big\}$.
This problem is NP-hard to solve, so \cite{candes2009robust} replaces the rank and cardinality ($\|\cdot\|_0$) functions with their convex surrogates, the nuclear norm and the $L_1$ norm respectively, and solves the following convex optimization problem:
\begin{equation}\label{eq:rpca}
  \min_{\bX,\bE} \big\{ \|\bX\|_* + \lambda\|\bE\|_1 \;|\; \bX + \bE = \bB \big\}.
\end{equation}
It has been shown that the optimal solution to problem \eqref{eq:rpca} exactly recovers the low-rank matrix from arbitrary corruptions as long as the errors $\bE$ are sufficiently sparse relative to the rank of $\bX$, or more precisely, when the following bounds hold \cite{candes2009robust}:
    $\textrm{rank}(\bX) \leq \frac{\rho_r\max(n,m)}{\mu(\log\min(n,m))^2}, \quad \|\bE\|_0 \leq \rho_s mn$,
where $\rho_r$ and $\rho_s$ are positive constants, and $\mu$ is the incoherence parameter.

\section{Higher-order RPCA (Robust Tensor Recovery)}\label{sec:horpca}
For noisy tensor data subject to outliers and arbitrary corruptions, it is desirable to be able to exploit the structure in all dimensions of the data.  A direct application of RPCA essentially considers the low-rank structure in only one of the unfoldings and is often insufficient.  Hence, we need a model that is directly based on tensors.  To generalize RPCA to tensors, we regularize the Tucker rank Trank($\sbX$) of a tensor $\sbX$, leading to the following \textit{tensor PCP} optimization problem:
  $\min_{\sbX,\sbE} \big\{\textrm{Trank}(\sbX) + \lambda\|\sbE\|_0, s.t. \quad \sbX + \sbE = \sbB\big\}$.  This problem is NP-hard to solve, so we replace Trank($\sbX$) by the convex surrogate CTrank$(\sbX)$, and $\|\sbE\|_0$ by $\|\sbE\|_1$ to make the problem tractable:
\begin{equation}\label{eq:ho-rpca}
  \min_{\sbX,\sbE} \big\{ \textrm{CTrank}(\sbX) + \lambda\|\sbE\|_1 \;|\; \sbX + \sbE = \sbB \big\}.
\end{equation}
We call the model \eqref{eq:ho-rpca} \textit{Higher-order RPCA} (HoRPCA).  In the subsequent sections, we consider variations of problem \eqref{eq:ho-rpca} and develop efficient algorithms for solving these models.

\subsection{Singleton Model}\label{sec:singleton}
In the Singleton model, the tensor rank regularization term is the sum of the $N$ nuclear norms $\|\modeiX\|_*$ of the mode-$i$ unfoldings, $i = 1,\cdots, N$ of $\sbX$, i.e. $\textrm{CTrank}(\sbX) := \sum_i \|\modeiX\|_*$.  HoRPCA with the Singleton low-rank tensor model is thus the convex optimization problem
\begin{eqnarray}\label{eq:ho-rpca-single}
  \min_{\sbX,\sbE} \left\{ \sum_{i=1}^N \|\modeiX\|_* + \lambda_1\|\sbE\|_1 \;|\; \sbX + \sbE = \sbB \right\}.
\end{eqnarray}
This form of $\textrm{CTrank}(\sbX)$ was also considered in \cite{liu2009tensor,gandy2011tensor,tomioka2010estimation} for recovering low-rank tensors from partial observations of the data.
To solve problem \eqref{eq:ho-rpca-single}, we develop an alternating direction augmented Lagrangian (ADAL) method \cite{glowinski1975adal, gabay1976dual}\footnote{This class of algorithms is also known as the alternating direction method of multipliers (ADMM) \cite{boyd2010distributed,eckstein1992douglas}.} to take advantage of the problem structure.  Specifically, by applying variable-splitting (e.g. see \cite{bertsekas1989parallel,boyd2010distributed}) to $\sbX$ and introducing $N$ auxiliary variables $\sbX_1,\cdots,\sbX_N$, problem \eqref{eq:ho-rpca-single} is reformulated as
\begin{eqnarray}\label{eq:ho-rpca-c}
  \min_{\sbX_1,\cdots,\sbX_N,\sbE} \left\{ \sumOneToN \|\modeiXi\|_* + \lambda_1\|\sbE\|_1 \vbar \sbX_i + \sbE = \sbB, \quad i = 1,\cdots,N \right\}.
\end{eqnarray}
Note that equality among the $\sbX_i$'s is enforced implicitly by the constraints, so that an additional auxiliary variable as in \cite{gandy2011tensor} is not required.

Before we develop the ADAL algorithm for solving problem \eqref{eq:ho-rpca-c} (see Algorithm \ref{alg:adal-horpca} below), we need to define several operations.  $\textrm{fold}_i(\bX)$ returns the tensor $\sbZ$ such that $\modeiZ = \bX$.  $\mathcal{T}_\mu(\bX)$ is the matrix singular value thresholding operator: $\calT_\mu(\bX) := \bU\textrm{diag}(\bar{\sigma})\bV^T$, where $\bX = \bU\textrm{diag}(\sigma)\bV^T$ is the SVD of $\bX$ and $\bar{\sigma} := \max(\sigma-\mu,0)$. We define $\calT_{i,\mu}(\sbX) := \textrm{fold}_i(\calT_\mu(\modeiX))$.  $\mathcal{S}_\mu(\sbX)$ is the shrinkage operator on vec($\sbX$) and returns the result as a tensor.  The vector shrinkage operator is defined as $\calS_\mu(\bx) := \textrm{sign}(\bx)\max(|\bx|-\mu,0)$, where the operations are all element-wise.

Problem \eqref{eq:ho-rpca-c} is in the generic form
\begin{eqnarray}\label{eq:adal_generic}
    \min_{x,y} \big\{ f(x) + g(y) \vbar Ax = y \big\},
\end{eqnarray}
which has the augmented Lagrangian
    $\calL(x,y,\gamma) := f(x) + g(y) + \gamma^T(y-Ax) + \frac{\mu}{2}\|Ax-y\|_2^2$.
\begin{algorithm}
\caption{ADAL}
\begin{algorithmic}[1]\label{alg:adal}
\small{
\STATE Choose $\gamma^0$.
\FOR{$k = 0,1,\cdots,K$}
    \STATE $x^{k+1} \gets \arg\min_{x}\mathcal{L}(x, y^k, \gamma^k )$
    \STATE $y^{k+1} \gets \arg\min_{y}\mathcal{L}(x^{k+1}, y, \gamma^k )$
    \STATE $\gamma^{k+1} \gets \gamma^k - \mu(Ax^{k+1} - y^{k+1})$
\ENDFOR
\RETURN $y^K$
}
\end{algorithmic}
\end{algorithm}
In \cite{eckstein1992douglas}, Eckstein and Bertsekas proved Theorem \ref{thm:admm_conv} below, establishing the convergence of the ADAL, Algorithm \ref{alg:adal}:
\begin{thm}\label{thm:admm_conv}
Consider problem \eqref{eq:adal_generic}, where both $f$ and $g$ are proper, closed, convex functions, and $A \in \mathbb{R}^{n\times m}$ has full column rank.  Then, starting with an arbitrary $\mu > 0$ and $x^0,y^0 \in \mathbb{R}^m$, the sequence $\{x^k,y^k,\gamma^k\}$ generated by Algorithm \ref{alg:adal} converges to a Kuhn-Tucker pair $\big((x^*,y^*),\gamma^*\big)$ of problem \eqref{eq:adal_generic}, if \eqref{eq:adal_generic} has one.  If \eqref{eq:adal_generic} does not have an optimal solution, then at least one of the sequences $\{(x^k,y^k)\}$ and $\{\gamma^k\}$ diverges.
\end{thm}

The augmented Lagrangian for problem \eqref{eq:ho-rpca-c} is
\begin{equation*}
    \sumOneToN \|\modeiXi\|_* + \lambda_1\|\sbE\|_1 + \sumOneToN \left(\frac{1}{2\mu}\|\sbX_i+\sbE-\sbB\|^2 - \langle \sbLambda_i, \sbX_i+\sbE-\sbB \rangle\right).
\end{equation*}
Observe that given $\sbE$, the $\sbX_i$'s can be solved for independently by the singular value thresholding operator.  Conversely, with fixed $\sbX_i$'s, the augmented Lagrangian subproblem with respect to $\sbE$ has a closed-form solution as we now show.  The subproblem under consideration is
    $\min_{\sbE} \left\{ \frac{1}{2}\sumOneToN\| \sbE + \sbX_i - \sbB - \mu\sbLambda_i \|^2 + \mu\lambda_1\|\sbE\|_1 \right\}$,
which can be written concisely as
\begin{equation}\label{eq:horpca-Esubprob}
    \min_{\sbE} \quad \frac{1}{2}\| \calC(\sbE) + \sbD \|^2 + \mu\lambda_1\|\sbE\|_1,
\end{equation}
where $\calC(\sbE) := \textrm{TArray}(\sbE,\cdots,\sbE), \textrm{and} \; \sbD = \textrm{TArray}(\sbX_1-\sbB-\mu\sbLambda_1,\cdots,\sbX_N-\sbB-\mu\sbLambda_N)$.
The following proposition shows that Problem \eqref{eq:horpca-Esubprob} has a closed-form solution.

\begin{prop}\label{prop:horpca-Esub-sol}
Problem \eqref{eq:horpca-Esubprob} is equivalent to
\begin{equation}\label{eq:horpca-Esubprob-easy}
    \min_{\sbE} \quad \frac{1}{2}\| \sbE + \frac{\calC^*(\sbD)}{N} \|^2 + \frac{\mu\lambda_1}{N}\|\sbE\|_1,
\end{equation}
which has a closed-form solution $\sbE = \calS_{\frac{\mu\lambda_1}{N}}\left(-\frac{1}{N}\sumOneToN(\sbX_i-\sbB-\mu\sbLambda_i)\right)$.
\end{prop}
\begin{proof}
Since $\calC^*\calC(\sbE) = N\sbE$, the first-order optimality conditions for \eqref{eq:horpca-Esubprob} are
  $0 \in N\sbE + \calC^*(\sbD) + \mu\lambda_1\partial\|\sbE\|_1
  \Leftrightarrow 0 \in \sbE + \frac{\calC^*(\sbD)}{N} + \frac{\mu\lambda_1}{N}\partial\|\sbE\|_1$,
which are the optimality conditions for \eqref{eq:horpca-Esubprob-easy}.
\end{proof}

In Algorithm \ref{alg:adal-horpca}, we state our ADAL algorithm that alternates between two blocks of variables, $\{\sbX_1,\cdots,\sbX_N\}$ and $\sbE$.  By defining $f(\sbX_1,\cdots,\sbX_N) := \sumOneToN\|\modeiXi\|_*$ and $g(\sbE) := \lambda_1\|\sbE\|_1$, it is easy to verify that problem \eqref{eq:ho-rpca-c} and Algorithm \ref{alg:adal-horpca} satisfy the conditions in Theorem \ref{thm:admm_conv}.  Hence, the convergence of Algorithm \ref{alg:adal-horpca} follows from Theorem \ref{thm:admm_conv}:
\begin{cor}\label{thm:horpca_conv}
The sequence $\{\sbX_1\supIndk,\cdots,\sbX_N\supIndk,\sbE\supIndk\}$ generated by Algorithm \ref{alg:adal-horpca} converges to an optimal solution $(\sbX_1^*,\cdots,\sbX_N^*,\sbE^*)$ of problem \eqref{eq:ho-rpca-c}.  Hence, the sequence $\{\frac{1}{N}(\sumOneToN\sbX_i\supIndk),\sbE\supIndk\}$ converges to an optimal solution of HoRPCA with the Singleton model \eqref{eq:ho-rpca-single}.
\end{cor}

\begin{algorithm}
\caption{HoRPCA-S (ADAL)}
\begin{algorithmic}[1]\label{alg:adal-horpca}
\small{
\STATE Given $\sbB, \lambda, \mu$.  Initialize $\sbX_i\supInd{0} = \sbE\supInd{0} = \sbLambda_i\supInd{0} = \mathbf{0}, \forall i \in \{1,\cdots,N\}$.
\FOR{$k = 0,1,\cdots$}
    \FOR{$i = 1:N$}
        \STATE $\sbX_i\supIndkpone \gets \calT_{i,\mu}(\sbB + \mu\sbLambda_i\supIndk - \sbE\supIndk)$
    \ENDFOR
    \STATE $\sbE\supIndkpone \gets \mathcal{S}_{\frac{\mu\lambda_1}{N}}(-\frac{\calC^*(\sbD)}{N})$
    \FOR{$i = 1:N$}
        \STATE $\sbLambda_i\supIndkpone \gets \sbLambda_i\supIndk - \frac{1}{\mu}(\sbX_i\supIndkpone + \sbE\supIndkpone - \sbB)$
    \ENDFOR
\ENDFOR
\RETURN $\left(\frac{1}{N}(\sumOneToN\sbX_i\supIndK), \sbE\supIndK\right)$
}
\end{algorithmic}
\end{algorithm}

\subsection{Mixture Model}
The Mixture model for a low-rank tensor was introduced in \cite{tomioka2010estimation}, which only requires that the tensor be the sum of a set of component tensors, each of which is low-rank in the corresponding mode, i.e. $\sbX = \sumOneToN \sbX_i$, where $\modeiXi$ is a low-rank matrix for each $i$.  This is a relaxed version of the Singleton model, which requires that the tensor be low-rank in all modes simultaneously.  It is shown in \cite{tomioka2010estimation} that the Mixture model is able to automatically detect the rank-deficient modes and yields better recovery performance than the Singleton model for tensor completion tasks when the original tensor is low-rank only in certain modes.

For robust tensor recovery, the Mixture model is equally applicable to represent the low-rank component of a corrupted tensor.  Specifically, we solve the convex optimization problem
\begin{eqnarray}\label{eq:horpca-mix}
  \min_{\sbX,\sbE} \left\{ \sum_{i=1}^N \|\modeiXi\|_* + \lambda_1\|\sbE\|_1 \vbar \sumOneToN\sbX_i + \sbE = \sbB \right\}.
\end{eqnarray}
This is a more difficult problem to solve than \eqref{eq:ho-rpca-single}; while the subproblem with respect to $\sbE$ still has a closed-form solution involving the shrinkage operator $\calS$, the $\sbX_i$ variables are coupled in the constraint, and it is hard to develop an efficient ADAL algorithm with two-block updates that satisfies the conditions in Theorem \ref{thm:admm_conv}.  Motivated by the approximation technique used in \cite{yang2009alternating}, we propose an inexact ADAL algorithm to solve problem \eqref{eq:horpca-mix} with global convergence guarantee.

Consider the augmented Lagrangian subproblem of \eqref{eq:horpca-mix} with respect to $\setOneToN{\sbX_i}$, which can be simplified to
\begin{equation}\label{eq:horpca-mix-Xsubprob}
    \min_{\sbX_1,\cdots,\sbX_N} \quad \frac{1}{2}\| \sumOneToN\sbX_i + \sbE - \sbB - \mu\sbLambda \|^2 + \mu\sumOneToN\|\modeiXi\|_*.
\end{equation}
Let $\calSigma\equiv\calC^*:\bbR^{N\cdot I_1\times\cdots\times I_N} \rightarrow \bbR^{I_1\times\cdots\times I_N}
$ be the summation operator, $\sbD = \sbE - \sbB - \mu\sbLambda$, and $f(\sbXbar) := \frac{1}{2}\|\calSigma(\sbXbar) + \sbD\|^2$, with $\nabla f(\sbXbar) = \calSigma^*(\calSigma(\sbXbar) + \sbD)$.  Then, problem \eqref{eq:horpca-mix-Xsubprob} can be written as
    $\min_{\sbX_1,\cdots,\sbX_N} \left\{ f(\sbXbar) + \mu\sumOneToN\|\modeiXi\|_* \right\}$.
This problem is not separable in $\setOneToN{\sbX_i}$, and an iterative method (e.g. \cite{ma2009fixed}) has to be used to solve this problem.  Instead, we consider the proximal approximation of the problem and solve, given $\sbXbar\supIndk \equiv \textrm{TArray}(\sbX_1\supIndk,\cdots,\sbX_N\supIndk)$,
\begin{equation}\label{eq:horpca-mix-Xproximal}
    \min_{\sbXbar} \quad f(\sbXbar\supIndk) + \nabla f(\sbXbar\supIndk)^T(\sbXbar - \sbXbar\supIndk) + \frac{1}{2\eta}\|\sbXbar - \sbXbar\supIndk\|^2 + \mu\sumOneToN\|\modeiXi\|_*,
\end{equation}
where $\eta$ is a user-supplied parameter which is less than $\lambda_{\textrm{max}}(\calSigma^*\calSigma)$ (equal to $N$ in this case).
It is easy to see that problem \eqref{eq:horpca-mix-Xproximal} is separable in $\setOneToN{\sbX_i}$, and the optimal solution has a closed-form solution $\sbX_i = \calT_{i,\eta\mu}(p(\sbXbar\supIndk)_i), \quad i = 1,\cdots,N$, where $p(\sbXbar\supIndk) = \sbXbar\supIndk - \eta\nabla f(\sbXbar\supIndk)$.  We state this inexact ADAL method below in Algorithm \ref{alg:adal-horpca-mix}.  The convergence of the algorithm is guaranteed by:
\begin{thm}\label{thm:horpca_m_conv}
For any $\eta$ satisfying $0 < \eta < \frac{1}{N}$, the sequence $\{\sbXbar\supIndk,\sbE\supIndk\}$ generated by Algorithm \ref{alg:adal-horpca-mix} converges to the optimal solution $\{\sbXbar^*,\sbE^*\}$ of problem \eqref{eq:horpca-mix}.
\end{thm}

\begin{proof}
See Appendix \ref{sec:app_conv_iadal}.
\end{proof}

\begin{algorithm}
\caption{HoRPCA-M (I-ADAL)}
\begin{algorithmic}[1]\label{alg:adal-horpca-mix}
\small{
\STATE Given $\sbB, \lambda, \mu$. Set $\eta = \frac{1}{N+1}$.  Initialize $\sbX_i\supInd{0} = \sbE\supInd{0} = \sbLambda\supInd{0} = \mathbf{0}, \forall i \in \{1,\cdots,N\}$.
\FOR{$k = 0,1,\cdots$}
    \FOR{$i = 1:N$}
        \STATE $\sbX_i\supIndkpone \gets \calT_{i,\eta\mu}(p(\sbXbar\supIndk)_i)$
    \ENDFOR
    \STATE $\sbE\supIndkpone \gets \mathcal{S}_{\mu\lambda_1}(\sbB + \mu\sbLambda\supIndk - (\sumOneToN\sbX_i\supIndkpone))$
\ENDFOR
\STATE $\sbLambda\supIndkpone \gets \sbLambda\supIndk - \frac{1}{\mu}(\sumOneToN\sbX_i\supIndkpone + \sbE\supIndkpone - \sbB)$
\RETURN $(\sumOneToN\sbX_i\supIndK, \sbE\supIndK)$
}
\end{algorithmic}
\end{algorithm}

\begin{rem}\label{rem:adaptive_single}
If we have prior knowledge of which modes of $\sbX$ are low-rank and which modes are not, we can also use an adaptive version of the Singleton model which allows for different regularization weights to be applied to the nuclear norms:
\begin{eqnarray}\label{eq:horpca-single-adapt}
  \min_{\sbX,\sbE} \left\{ \sumOneToN \lambda_{*i}\|\modeiX\|_* + \lambda\|\sbE\|_1 \vbar \sbX + \sbE = \sbB \right\}.
\end{eqnarray}
The downside of using this more general version of the Singleton model over the Mixture model is that there are more parameters to tune, and in general, we do not have the rank information \textit{a priori}.  In our experiments in Section \ref{sec:exp}, we examined the results from the non-adaptive Singleton model, which usually shed some light on the rank of the tensor modes.  We then adjusted the $\lambda_{*i}$'s accordingly.  Note that RPCA applied to the mode-$n$ unfolding of a given tensor is a special case of \eqref{eq:horpca-single-adapt} where $\lambda_{*i} = 0$ for all $i \neq n$.  When the tensor is low-rank in only one mode, RPCA applied to that particular mode should be able to achieve good recovery results, but it may be necessary to apply RPCA to every unfolding of the tensor to discover the best mode, resulting in a similar number of SVD operations as HoRPCA.
\end{rem}

\subsection{Unconstrained (Lagrangian) Version}
In the Lagrangian version of HoRPCA, the consistency constraints are relaxed and appear as quadratic penalty terms in the objective function.  For the Singleton model, we solve the optimization problem
\begin{equation}\label{eq:horpca-unc}
    \min_{\sbX_1,\cdots,\sbX_N,\sbE} \quad \frac{1}{2}\sumOneToN\|\sbX_i + \sbE - \sbB\|^2 + \lambda_*\sumOneToN \|\modeiXi\|_* + \lambda_1\|\sbE\|_1.
\end{equation}
Similarly, the Lagrangian optimization problem for the Mixture model is
\begin{equation}\label{eq:horpca-mix-unc}
    \min_{\sbX_1,\cdots,\sbX_N,\sbE} \quad \frac{1}{2}\|\sumOneToN\sbX_i + \sbE - \sbB\|^2 + \lambda_*\sumOneToN \|\modeiXi\|_* + \lambda_1\|\sbE\|_1.
\end{equation}
Taking advantage of properties of these problems, we develop an accelerated proximal gradient (APG) algorithm by applying FISTA \cite{beckteboulle} with continuation to solve them.  Before stating this algorithm below in Algorithm \ref{alg:fista-horpca}, we define the following notation.
$\sbXtild :=  \textrm{TArray}(\sbX_1,\cdots,\sbX_N,\sbE)$, and
    $\calA(\sbXtild) := \textrm{TArray}(\sbX_1 + \sbE,\cdots,\sbX_N + \sbE)$, $\calA^*(\sbXbar) :=  \textrm{TArray}(\sbX_1,\cdots,\sbX_N,\sumOneToN\sbX_i)$ and $\tilde{\sbB} = \textrm{TArray}(\sbB,\cdots,\sbB)$, for the Singleton model, and
$\calA := \calSigma$, the summation operator defined on $\sbXtild$, and $\tilde{\sbB} = \sbB$, for the Mixture model.  The global convergence rate is discussed in Appendix \ref{sec:app_fista}.

\begin{algorithm}
\caption{HoRPCA-SP/HoRPCA-MP (FISTA with continuation)}
\begin{algorithmic}[1]\label{alg:fista-horpca}
\small{
\STATE Given $\sbXtild^{(0)}, \lambda_*^0, \bar{\lambda_*}, \eta$ and $r$. Set $t_0 = 0, \sbYtild^{(0)} = \sbXtild^{(0)}$, and $\mu = \frac{1}{N+1}$.
\FOR{$k = 0,1,\cdots$ until convergence}
    \STATE $\lambda_1 \gets r\lambda_*$
    \STATE $\sbPtild \gets \sbYtild\supIndk - \mu\calA^*(\calA(\sbYtild\supIndk - \sbBtild))$
    \FOR{$i = 1:N$}
        \STATE $\sbXtild_i\supIndkpone \gets \calT_{i,\mu\lambda_*}(\sbP_i)$
    \ENDFOR
    \STATE $\sbE\supIndkpone \gets \mathcal{S}_{\mu\lambda_1}(\sbP_E)$
    \STATE $t_{k+1} \gets \frac{1 + \sqrt{1+4t_k^2}}{2}$
    \STATE $\sbYtild\supIndkpone \gets \sbXtild\supIndk + \frac{t_k - 1}{t_{k+1}}(\sbXtild\supIndkpone - \sbXtild\supIndk)$
    \STATE $\lambda_* \gets \max(\eta\lambda_*, \bar{\lambda_*})$
\ENDFOR
}
\end{algorithmic}
\end{algorithm}

\subsection{Partial Observations}\label{sec:missing-data}
When some of the data is missing, we can enforce the consistency on the observable data through the projection operator $\calAc: \bbR^{I_1\times\cdots\times I_N} \rightarrow \bbR^m$ that selects the set of $m$ observable elements ($\Omega$) from the data tensor.  Assuming that the support of the true error tensor $\sbE_0$, $\Pi$ is a subset of $\Omega$, which implies that $[\sbE_0]_{\bar{\Omega}} = 0$, we essentially have to solve the same optimization problems as above, with a minor modification to the consistency constraint.  The optimization problem corresponding to the Singleton model becomes
\begin{eqnarray}\label{eq:horpca-tc}
  \min_{\sbX_1,\cdots,\sbX_N,\sbE} &\;& \sumOneToN \|\modeiXi\|_* + \lambda_1\|\sbE\|_1 \\
   \nonumber s.t. &\;& \calAc(\sbX_i + \sbE) = \bB_\Omega, \; i = 1,\cdots,N, \\
   \nonumber &\;& \sbX_1=\sbX_2=\cdots=\sbX_N ,
\end{eqnarray}
and that corresponding to the Mixture model is
\begin{eqnarray}\label{eq:horpca-mix-tc}
  \min_{\sbX,\sbE} \left\{ \sum_{i=1}^N \|\modeiXi\|_* + \lambda_1\|\sbE\|_1 \vbar \calAc(\sumOneToN\sbX_i + \sbE) = \bB_\Omega \right\}.
\end{eqnarray}
Note that without the additional assumption that $[\sbE_0]_{\bar{\Omega}} = 0$, it is impossible to recover $\sbE_0$ since some of corrupted tensor elements are simply not observed.  The $L_1$-regularization of $\sbE$ in the objective function forces any element of an optimal solution $\sbE^*$ in $\bar{\Omega}$ to be zero.  In the two-dimensional case, this corresponds to the ``Matrix Completion with Corruption'' model in \cite{candes2009robust,li2011compressed}, and exact recoverability through solving a two-dimensional version of \eqref{eq:horpca-tc} is guaranteed under some assumptions on the rank, the sparsity, and the fraction of data observed \cite{li2011compressed}.

For the Mixture model, we can apply Algorithm \ref{alg:adal-horpca-mix} with only minor modifications.  Specifically, we replace $\calSigma$ with $\calAc\circ\calSigma$ and $\sbD$ with $\calAc(\sbD)$ in the definition of $\nabla f(\cdot)$ and hence $p(\cdot)$.\footnote{Note that the symbol $\circ$ here denotes the pipeline of two linear operators instead of the vector outer product defined in Section \ref{sec:notation}.}  The subproblem with respect to $\sbE$ becomes (after simplification)
\begin{equation}\label{eq:horpca-mix-tc-Esubprob}
    \min_{\sbE}\quad \frac{1}{2}\|\calAc(\sbE) + \calAc(\sumOneToN\sbX_i - \sbB - \mu\sbLambda)\|^2 + \mu\lambda_1\|\sbE\|_1.
\end{equation}
The following proposition shows that the above problem admits the closed-form solution \\
    $\sbE = \mathcal{S}_{\mu\lambda_1}\left(\calAc^*\calAc(\sbB + \mu\sbLambda - (\sumOneToN\sbX_i))\right)$.
\begin{prop}
The optimal solution of the optimization problem
\begin{equation}\label{eq:horpca-mix-tc-Esubprob-short}
    \min_{\sbE}\quad \frac{1}{2}\|\calAc(\sbE) - \bd\|^2 + \lambda_1\|\sbE\|_1
\end{equation}
is $\sbE^* = \calS_{\lambda_1}(\calAc^*(\bd))$.
\end{prop}
\begin{proof}
Let $\sbE_\Omega = \calAc(\sbE)$, the vector of the elements of $\sbE$ whose indices are in the observation set $\Omega$ and $\sbE_{\bar{\Omega}}$ be the vector of the remaining elements.  We can then write problem \eqref{eq:horpca-mix-tc-Esubprob-short} as
\begin{equation}\label{eq:horpca-mix-tc-Esubprob-decomp}
    \min_{\sbE}\quad \frac{1}{2}\|\sbE_\Omega - \bd\|^2 + \lambda_1(\|\sbE_\Omega\|_1 + \|\sbE_{\bar{\Omega}}\|_1),
\end{equation}
which is decomposable.  Obviously, $\sbE_{\bar{\Omega}}^* = 0$.  The optimal solution to problem \eqref{eq:horpca-mix-tc-Esubprob-decomp} with respect to $\sbE_\Omega$ is given by the shrinkage operator, $\sbE_\Omega^* = \calS_{\lambda_1}(\bd)$.  Hence, $\sbE^* = \calS_{\lambda_1}(\calAc^*(\bd))$, and we obtain the optimal solution to problem \eqref{eq:horpca-mix-tc-Esubprob} by setting $\bd = -\calAc(\sumOneToN\sbX_i - \sbB - \mu\sbLambda)$.
\end{proof}

The proximal approximation to the subproblem with respect to $\setOneToN{\sbX_i}$ is still separable, and each $\sbX_i$ can be solved for by applying the singular value thresholding operator.

For the Singleton model, we introduce an auxiliary variable $\sbY$ and reformulate problem \eqref{eq:horpca-mix} as
\begin{eqnarray}\label{eq:horpca-tc-split}
  \min_{\sbX_1,\cdots,\sbX_N,\sbY,\sbE} &\;\;& \sum_{i=1}^N \|\modeiXi\|_* + \lambda_1\|\sbE\|_1 \\
  \nonumber s.t. &\;\;& \sbX_i = \sbY, \quad i = 1,\cdots,N,\\
  \nonumber     &\;\;& \calAc(\sbY + \sbE) = \sbB_\Omega.
\end{eqnarray}
We can then develop an ADAL algorithm that employs two-block updates between $\{\sbX_1,\cdots,\sbX_N,\sbE\}$ and $\sbY$. The solutions to $\setOneToN{\sbX_i}$ still admit closed-form expressions by applying the singular value thresholding operator.  The solution to $\sbE$ is a similar form as the one for the Mixture model.  The augmented Lagrangian subproblem with respect to $\sbY$ involves solving a linear system which also has a closed-form solution.

For the two Lagrangian versions of HoRPCA, the required modifications are minimal.  We simply need to redefine the linear operator $\calA$ in Algorithm \ref{alg:fista-horpca} to be $\calAc\circ\calA$ and apply $\calAc$ to $\sbBtild$.  It can also be verified that the Lipschitz constant of $\nabla l(\cdot)$ is still $N+1$ for both models.


\subsection{Related Work}
Several methods have proposed for solving the RPCA problem, including the Iterative Thresholding algorithm \cite{wright2009robust}, the Accelerated Proximal Gradient (APG/FISTA) algorithm with continuation \cite{lin2009fast} for the Lagrangian formulation of \eqref{eq:rpca}, a gradient algorithm applied to the dual problem of \eqref{eq:rpca}, and the Inexact Augmented Lagrangian method (IALM) in \cite{lin2010augmented}.  It is reported in \cite{lin2010augmented} that IALM was faster than APG on simulated data sets.

For the unconstrained formulation of Tensor Completion with the Singleton model,
\begin{equation}\label{eq:tc_unc}
    \min_{\sbX} \quad \lambda_*\sumOneToN \|\modeiX\|_* + \frac{1}{2}\|\calAc(\sbX) - \sbB_\Omega)\|^2,
\end{equation}
\cite{gandy2011tensor} and \cite{tomioka2010estimation} both proposed an ADAL algorithm based on applying variable-splitting on $\sbX$.  For the Mixture model version of \eqref{eq:tc_unc}, \cite{tomioka2010estimation} also proposed an ADAL method applied to the dual problem.

There have been some attempts to tackle the HoRPCA problem \eqref{eq:ho-rpca-single} with applications in computer vision and image processing.  The RSTD algorithm proposed in \cite{li2010optimum} uses a vanilla Block Coordinate Descent (BCD) approach to solve the unconstrained problem
\begin{equation*}
    \min_{\sbX,\sbE,\{\bM_i,\bN_i\}_i} \; \frac{1}{2}\sumOneToN(\alpha_i\|\bM_i - \modeiX\|^2 + \beta_i\|\bN_i - \modeLetter{\bE}{i}\|^2 + \gamma_i\|\bM_i+\bN_i-\modeLetter{\bB}{i}\|^2) + \sumOneToN( \lambda_i\|\bM_i\|_* + \eta\|\bE\|_1 ),
\end{equation*}
which applies variable-splitting to both $\sbX$ and $\sbE$ and relaxes all the constraints as quadratic penalty terms.  Compared to HoRPCA-SP, RSTD has many more parameters to tune. Moreover, the BCD algorithm used in RSTD does not have a iteration complexity guarantee enjoyed by FISTA.  (More sophisticated variants of BCD do have iteration complexity results, e.g. \cite{nesterov2012efficiency,richtarik2012iteration,beck2011convergence}.)

The TR-MALM algorithm proposed in \cite{tan2011tensor} is also an ADAL method and solves a relaxed version of \eqref{eq:ho-rpca-c}:
\begin{eqnarray}\label{eq:horpca_trmalm}
    \min_{\{\sbX_i,\sbE_i\}_i} \left\{ \sumOneToN\left(\|\modeiXi\|_* + \lambda_i\|\modenn{\bE}{i}{i}\|_1\right) \vbar \modeiXi + \modenn{\bE}{i}{i} = \modeLetter{\bB}{i}, \quad i = 1,\cdots,N \right\}.
\end{eqnarray}
The final solution is given by $\left( \frac{1}{N}\sumOneToN\sbX_i, \frac{1}{N}\sumOneToN\sbE_i \right)$.
Compared to problem \eqref{eq:ho-rpca-c}, problem \eqref{eq:horpca_trmalm} does not require equality among the auxiliary variables $\sbX_i$'s and $\sbE_i$'s.  This relaxation allows the problem to be decomposed into $N$ independent RPCA instances, each of which solvable by IALM mentioned above.  However, this relaxation also makes the final solution hard to interpret since consistency among the auxiliary variables is not guaranteed.

The $l_1$-regularization used in HoRPCA is a special instance of the exponential family used in \cite{hayashi2010exponential}.  That approach is, nevertheless, very different from ours. Before learning the model proposed in \cite{hayashi2010exponential}, the (heterogeneous) distributions of parts of the data have to be specified, making the method hard to apply.  In our case, the locations of the corrupted entries are not known in advance.

Another related regularization technique for robustifying tensor factorization is the $R_1$ norm used in \cite{huang2008robust}.  In fact, the $R_1$ norm is the group Lasso \cite{yuanlin} regularization with each slice (in the case of a 3D tensor) defined as a group, and the regularization is applied to the error term of the tensor factorization.  Hence, it is most effective on data in which dense corruptions are concentrated in a small number of slices, i.e. outlying samples.

\subsection{Relationship with the Tucker Decomposition}\label{sec:relationship-Tucker}


We can also view HoRPCA with the Singleton model as a method for \textit{Robust Tensor Decomposition}, because from the auxiliary variables $\sbX_i$'s, we can reconstruct the core tensor $\sbG$ of $\sbX$ by
    $\sbG = \frac{1}{N}(\sumOneToN\sbX_i) \times_1 (\bU\supInd{1})^T \times_2 (\bU\supInd{2})^T \cdots \times_N (\bU\supInd{N})^T$,
where $\bU\supInd{i}$ is the left factor matrix from the SVD of $\modeiX$.  Note that the $\bU\supInd{i}$ matrices are by-products of Algorithms \ref{alg:adal-horpca} and \ref{alg:fista-horpca}.  Hence, we recover the Tucker decomposition of $\sbX$ containing sparse arbitrary corruptions without the need to specify the target Tucker rank.  The CP decomposition of the low-rank tensor can also be obtained from the output of Algorithms \ref{alg:adal-horpca} and \ref{alg:fista-horpca} by applying the classical CP to the core tensor \cite{tomioka2010estimation}.  

\section{Constrained Nonconvex Model}\label{sec:ncx}
We consider a robust model that has explicit constraints on the Tucker rank of the tensor to be recovered.  Specifically, we solve the (nonconvex) optimization problem
\begin{equation}\label{eq:horpca-ncx}
  \min_{\sbX,\sbE} \;\; \bigg\{ \|\sbE\|_1 \;|\;  \sbX + \sbE = \sbB, \;\; \textrm{rank}(\modeiX) \leq r_i, \quad i = 1,\cdots,N \bigg\}.
\end{equation}
The global optimal solution to problem \eqref{eq:horpca-ncx} is generally NP-hard to find.  Here, we develop an efficient algorithm based on ADAL that has the same per-iteration complexity as Algorithm \ref{alg:adal-horpca} and finds a reasonably good solution.

\subsection{Algorithm}
Applying the variable-splitting technique that we used for the Singleton model, we reformulate problem \eqref{eq:horpca-ncx} as
\begin{equation}\label{eq:horpca-ncx-split}
  \min_{\sbX,\sbE} \;\; \bigg\{ \|\sbE\|_1 \;|\; \sbX_i + \sbE = \sbB, \;\; \textrm{rank}(\modeiXi) \leq r_i, \quad i = 1, \cdots, N \bigg\},
\end{equation}
where $\sbX_i$'s are auxiliary variables.  Although this is a nonconvex problem, applying ADAL still generates well-defined subproblems which have closed-form global optimal solutions.

The augmented Lagrangian for problem \eqref{eq:horpca-ncx-split}, dualizing only the consistency constraints, is equivalent to (up to some constants)
$\mu\|\sbE\|_1 + \frac{1}{2}\|\sbX_i -\sbZ_i\|^2$, where $\sbZ_i = \sbB + \mu\sbLambda_i - \sbE$.
For $i = 1,\cdots,N$, the augmented Lagrangian subproblem associated with $\sbX_i$ is
\begin{equation}\label{eq:horpca-ncx-Xsubprob}
  \min_{\sbX_i} \;\; \bigg\{ \|\sbX_i - \sbZ_i\|^2 \;|\; \textrm{rank}(\modeiXi) \leq r_i \bigg\}.
\end{equation}
We can interpret this problem as finding the Euclidean projection of $\sbZ_i$ onto the set of matrices whose ranks are at most $r_i$. The global optimal solution is given by \begin{thm}
(Eckart-Young \cite{eckart1936approximation}) For any positive integer $k \leq \textrm{rank}(\bZ)$, where $\bZ$ is an arbitrary matrix, the best rank-$k$ approximation of $\bZ$, i.e. the global optimal solution to the problem $\min_{\bX} \;\; \bigg\{ \|\bX-\bZ\|_F^2 \;|\; \textrm{rank}(\bX) \leq k \bigg\}$
is given by $\bX^* = \calP_k(\bZ) := \bU\bD_k\bV\transp$, where $\bD_k$ is a diagonal matrix consisting of the $k$ largest diagonal entries of $\bD$, given the SVD of $\bZ = \bU\bD\bV\transp$.
\end{thm}

Hence, we obtain the global minimizer of problem \eqref{eq:horpca-ncx-Xsubprob} by the truncated SVD of $\modeiZ$, $\calP_{k_i}(\modeiZ)$, which can be computed
efficiently if $k_i$ is small relative to rank($\modeiZ$).

The augmented Lagrangian subproblem associated with $\sbE$ has the same form as it has for the Singleton model.  By Proposition \ref{prop:horpca-Esub-sol}, we obtain the solution by $\frac{1}{N}\calS_\mu(\sumOneToN\sbB+\mu\sbLambda_i-\sbX_i)$.  Now, we summarize the main steps in Algorithm \ref{alg:horpca-ncx} (HoRPCA-C).  For the extension to the partial data case, we can apply variable-splitting as in \eqref{eq:horpca-tc-split} to derive a similar ADAL algorithm.

\begin{algorithm}
\caption{HoRPCA-C}
\begin{algorithmic}[1]\label{alg:horpca-ncx}
\small{
\STATE Given $\sbB, \mu$.  Initialize $\sbX_i\supInd{0} = \sbE\supInd{0} = \sbLambda_i\supInd{0} = \mathbf{0}, \forall i \in \{1,\cdots,N\}$.
\FOR{$k = 0,1,\cdots$}
    \FOR{$i = 1:N$}
        \STATE $\bZ_i\supIndk \gets \unfold_i(\sbB + \mu\sbLambda_i\supIndk - \sbE\supIndk)$
        \STATE $\sbX_i\supIndkpone \gets \fold_i(\calP_{k_i}(\bZ_i\supIndk))$ \label{line:horpca-c-X}
    \ENDFOR
    \STATE $\sbE\supIndkpone \gets \frac{1}{N}\calS_\mu(\sumOneToN\sbB+\mu\sbLambda_i\supIndk-\sbX_i\supIndkpone)$ \label{line:horpca-c-E}
    \FOR{$i = 1:N$}
        \STATE $\sbLambda_i\supIndkpone \gets \sbLambda_i\supIndk - \frac{1}{\mu}(\sbX_i\supIndkpone + \sbE\supIndkpone - \sbB)$ \label{line:horpca-c-Lambda}
    \ENDFOR
\ENDFOR
\RETURN $\left(\frac{1}{N}(\sumOneToN\sbX_i\supIndK), \sbE\supIndK\right)$
}
\end{algorithmic}
\end{algorithm}

\subsection{Convergence Analysis}
As we mentioned before, ADAL algorithms applied to a convex optimization problem that alternately minimizes two blocks of variables enjoy global convergence results.  Since problem \eqref{eq:horpca-ncx} is nonconvex, the existing convergence results no longer apply.  Here, we give a weak convergence result, that guarantees that whenever the iterates of Algorithm \ref{alg:horpca-ncx} converge, they converge to a KKT point of problem \eqref{eq:horpca-ncx}.  Although we do not have a proof for the convergence of the iterates, our experiments show strong indication of convergence.  To derive the KKT conditions, we first rewrite problem \eqref{eq:horpca-ncx} as
\begin{equation}\label{eq:horpca-ncx-factor}
  \min_{\sbE,\bU_i,\bV_i} \;\; \bigg\{ \|\sbE\|_1 \;|\; \fold_i(\bU_i\bV_i) + \sbE = \sbB, \quad i = 1,\cdots,N \bigg\},
\end{equation}
where $\bU_i$'s and $\bV_i$'s are matrices of appropriate dimensions with $r_i$ columns and rows respectively.  Let $\sbLambda_i, i = 1,\cdots,N,$ be the Lagrange multipliers associated with the $N$ constraints.  The KKT conditions for problem \eqref{eq:horpca-ncx-factor} are
\begin{eqnarray}
  \modeiLambdai\bV_i\transp &=& \bzero, \quad \bU_i\transp\modeiLambdai = \bzero \label{eq:ncx-kkt1}\\
  \fold_i(\bU_i\bV_i) + \sbE &=& \sbB, \quad i = 1,\cdots,N \label{eq:ncx-kkt2}\\
  \sumOneToN\sbLambda_i &\in& \partial\|\sbE\|_1. \label{eq:ncx-kkt3}
\end{eqnarray}

\begin{lem}\label{lem:ncx-conv}
Let $\sbW := (\sbX_1,\cdots,\sbX_N,\sbE,\sbLambda_1,\cdots,\sbLambda_N)$ and $\setkOneToInf{\sbW\supIndk}$ be the sequence generated by Algorithm \ref{alg:horpca-ncx}.  Assume that $\setkOneToInf{\sbW\supIndk}$ is bounded and $\sbW\supIndkpone - \sbW\supIndk \rightarrow 0$.  Then, any accumulation point of $\setkOneToInf{\sbW\supIndk}$ satisfies the KKT conditions \eqref{eq:ncx-kkt1}-\eqref{eq:ncx-kkt3}.
\end{lem}
\begin{proof}
See Appendix \ref{sec:app_conv_ncx}.
\end{proof}

\begin{rem}\label{rem:Tucker-adal}
From the discussion in Section \ref{sec:relationship-Tucker}, we can easily see that Algorithm \ref{alg:horpca-ncx} is indeed a variant of the Tucker decomposition that is robust to sparse gross corruptions.
Algorithm \ref{alg:horpca-ncx} also naturally leads to an alternative algorithm for the Tucker decomposition, which is traditionally solved by the ALS method as mentioned in the introduction.  Recall that the Tucker decomposition corresponds to the optimization problem
\begin{equation}\label{eq:tucker-problem}
  \min_{\sbX,\sbE} \;\; \bigg\{ \|\sbE\|^2 \;|\; \sbX + \sbE = \sbB, \;\; \textrm{rank}(\modeiX) \leq r_i, \quad i = 1,\cdots,N \bigg\}.
\end{equation}
By replacing Line \ref{line:horpca-c-E} in Algorithm \ref{alg:horpca-ncx} with $\sbE = \frac{1}{N+2\mu}\sumOneToN(\sbB+\mu\sbLambda_i-\sbX_i)$, which solves the augmented Lagrangian subproblem $\min_{\sbE}\;\left\{ \mu\|\sbE\|^2 + \frac{1}{2}\sumOneToN\|\sbE-\sbB-\mu\sbLambda_i+\sbX_i\|^2 \right\}$, we obtain an ADAL algorithm for the Tucker decomposition.  A similar convergence result to that given by Lemma \ref{lem:ncx-conv} applies to the derived algorithm.
\end{rem}

\begin{rem}\label{rem:horpca-c-rank-info}
Similar to the adaptive version of HoRPCA-S, HoRPCA-C requires prior information about the rank of the tensor unfoldings.  When this information is not available, we can apply HoRPCA-S or HoRPCA-M first to reveal the approximate tensor ranks.  HoRPCA-C can then be used as an effective post-processing step to refine the results.
\end{rem}

\section{Experiments}\label{sec:exp}
All the proposed algorithms and the experiments were run in Matlab R2011b on a laptop with a COREi5 2.40GHz CPU and 6G memory.  We used the PROPACK toolbox \cite{larsen2004propack} for SVD computations and the Tensor Toolbox \cite{bader2009efficient} for tensor operations and decompositions.  For running time comparisons, we report the number of iterations since the per-iteration work of all tensor-based algorithms involve $N$ SVDs and one shrinkage operation.  The amount of work performed by RPCA can also be estimated proportionally.

We considered two scenarios: 1) with fully observed data, and 2) with a fraction of data available.  The models and algorithms that we compared in our experiments are the Singleton model (HoRPCA-S and HoRPCA-SP), Mixture model (HoRPCA-M and HoRPCA-MP), and the constrained nonconvex model (HoRPCA-C).  RPCA (IALM) and TR-MALM were used as baselines in some experiments.  The number of iterations for TR-MALM was averaged over the $N$ RPCA instances.  A description of how the parameters $\lambda_1, \lambda_*, \textrm{and } \lambda_*^0$ were set for the algorithms can be found in Appendix \ref{sec:app_params} along with a discussion of stopping criteria.

\subsection{Synthetic Data}\label{sec:syn-data}
We generated a random rank-(5,5,5) tensor of size (50,50,20) using the approach described in \cite{tomioka2010estimation}, by drawing the core tensor of size (5,5,5) from the standard normal distribution and multiplying each mode of the core tensor by an orthonormal matrix of appropriate dimensions drawn from the Haar measure.\footnote{Generating random orthogonal matrices from the Haar measure is realized through calling the Matlab built-in function QMULT.}  All generated tensors were verified to have the desired Tucker rank.  A random fraction $\rho_n$ of the tensor elements were corrupted by additive i.i.d. noise from the uniform distribution $\calU(-M,M)$.  We then randomly selected a fraction $\rho_o$ of the noisy tensor elements to be the given observations $\sbB_\Omega$.

\subsubsection{Convex Models}
First, we fixed $\rho_n = 10\%$ and $M = 1$, and we applied the proposed algorithms described in Section \ref{sec:horpca} to the noisy tensor with $\rho_o$ varying from 5\% to 100\%.  We plot the recovery relative error $\frac{\|\sbX-\sbX_0\|}{\|\sbX_0\|}$, where $\sbX_0$ is the noiseless low-rank tensor, and the number of iterations required in Figure \ref{fig:syn-obs-n01}.  We repeated the experiment with $\rho_n = 25\%$ and plotted the results in Figure \ref{fig:syn-obs-n01}.  As a baseline, we ran Tucker decomposition on the data with $\rho_o = 100\%$ and report the results in the figure captions.

We see that HoRPCA-S had a significantly higher recovery accuracy than the other three algorithms.  In addition, for HoRPCA-S, there appeared to be a phase transition threshold in $\rho_o$ where the relative error of the algorithm improved drastically.  This threshold appeared to increase with $\rho_n$: it was about 35\% for $\rho_n = 10\%$, and it increased to 70\% for $\rho_n = 25\%$.  Interestingly, a similar phase transition phenomenon was not observed for the Lagrangian version HoRPCA-SP.

Next, we fixed $\rho_o$ at 100\% and varied $\rho_n$ from 1\% to 40\%.  The same set of metrics were reported and plotted in the left graph of Figure \ref{fig:syn-noise-obs100}.  Again, the Singleton model worked better than the Mixture model, which appeared to be very susceptible to increasing $\rho_n$'s.  There is also a threshold in $\rho_n$ (25\% in this case) below which HoRPCA-S could always recover the true low-rank tensor almost exactly.  The threshold for HoRPCA-SP was lower, at 15\%.

To further study the phase-transition behavior in the recovery performance of HoRPCA-S in terms of the fractions of observations and corruptions, we generated a set of 30 tensors of the same size (50,50,20) with different but low Tucker-ranks.  We consider exact recovery to be achieved when the relative error is less than or equal to 0.01. The results are summarized in Figure \ref{fig:syn-rank-spa-obs}.  In the first graph, we plot the threshold on the percentage of corruptions (as in Figure \ref{fig:syn-noise-obs100}) for each fraction of observations for three tensors of different Tucker-ranks.  Generally, as the Tucker rank sum increases, the level of corruptions that can be tolerated decreases.  For any given tensor, the transition from zero corruption tolerance to the maximum level of tolerance occurred over a small interval of observation percentages (e.g., from 50\% to 65\% observations for the rank-(2,5,5) tensor), and further availability of data did not improve the threshold on the percentage of corruptions.  The second graph plots the threshold on the percentage of corruptions for each tensor of a particular normalized rank \cite{tomioka2011statistical}, given a fixed fraction of observations.  Plots for three different fractions of observations (100\%, 80\%, and 60\%) are presented.  Generally, the threshold on the corruptions decreases approximately linearly as the normalized rank increases.  The recovery frontier shifts towards the right of the graph with a larger fraction of observations.  In the third graph, we investigated the minimum fraction of observations required for exact recovery (as in Figure \ref{fig:syn-obs-n01}) for each normalized rank, given a fixed percentage of corruptions.  Two different percentages of corruptions (1\% and 10\%) were considered.  Again, we observed an approximately linear relationship between the threshold on the observations and the normalized rank at both levels of corruptions.  This is consistent with the observation and the theory reported in \cite{tomioka2010estimation,tomioka2011statistical}.

We generated a second type of tensor of size (20,20,20,10) with Tucker rank (3,3,20,10), which is low-rank only in modes 1 and 2.  We fixed $\rho_n = 10\%$ and $M = 1$.  Our algorithms were applied to recover the low-rank tensor with varying $\rho_o$.  The right graph of Figure \ref{fig:syn-noise-obs100} shows the recovery results.  The Mixture model produced significantly smaller relative errors than the Singleton model.  However, we did not observe a phase transition threshold in $\rho_o$ for the Mixture model.  We also tested the adaptive version of HoRPCA-S \eqref{eq:horpca-single-adapt} with $\lambda_{*,1}$ and $\lambda_{*,2}$ set to 0.1.  The results were surprisingly good, and exact recovery was achieved with only 65\% of the data.  The experimental results on the synthetic data suggest that HoRPCA-S was the only algorithm that was able to achieve exact recovery with partial observations.  When the ground truth tensor was not low-rank in all modes, prior rank information was required for exact recovery.  When the rank information was not available, the Mixture model appeared to be a reasonable alternative with fully observed data.  The Tucker decomposition performed similarly to the non-adaptive Singleton model on fully observed data.

 In general, the ADAL-based algorithms were several times faster than the FISTA-based algorithms, and they achieved lower recovery error in most cases.  The difference in speed echoes an observation in \cite{lin2010augmented}.  The unconstrained formulation and the FISTA algorithms developed for them are more suitable for problems where data consistency is not required to be enforced strictly.  This general observation is similar to one made in \cite{lin2010augmented} for the RPCA problem.

\subsubsection{Nonconvex Model}
We repeated some of the experiments above using Algorithm \ref{alg:horpca-ncx} for different values of the Tucker-ranks to examine the recoverability property of the nonconvex constrained model \eqref{eq:horpca-ncx} (HoRPCA-C).  For the fully low-rank data, three Tucker-ranks were tested: [5,5,5], [6,6,6], and [8,8,8].  The values in the square parentheses correspond to $r_i$'s in \eqref{eq:horpca-ncx}.  We present recoverability results for (i) fixed $\rho_n = 25\%, M=1$ and varying $\rho_o$ and (ii) fixed $\rho_o = 100\%$ and varying $\rho_n$.  For the partially low-rank data, we tested the Tucker-ranks [3,3,20,10], [4,5,20,10], and [4,5,17,18] for fixed $\rho_n = 10\%, M=1$ and varying $\rho_o$.  The results are plotted in Figures \ref{fig:syn-ncx-lr} and \ref{fig:syn-ncx-plr}.

The first observation from these experimental results is that when the Tucker rank was correctly specified, HoRPCA-C yielded significantly better recovery performance in that it achieved near-exact recovery with much fewer observations (20\%) and was more robust to data corruption, up to 40\%.  In terms of speed, HoRPCA-C was comparable to HoRPCA-S, and the number of iterations required consistently decreased with increasing amount of observations (when the problem became easier and easier).  When the Tucker rank was over-estimated, the recovery performance of HoRPCA-C was impacted, but it was still comparable to that for HoRPCA-S.  On the other hand, there was a significant increase in the number of iterations.  Under-estimating the Tucker rank had a detrimental effect on the recovery performance, which is obvious from Figure \ref{fig:syn-ncx-plr}.

\begin{figure}
\begin{center}
    \hspace*{-0.1in}\includegraphics[trim = 0.6in 0in 0.5in 0in, clip, width=0.5\textwidth]{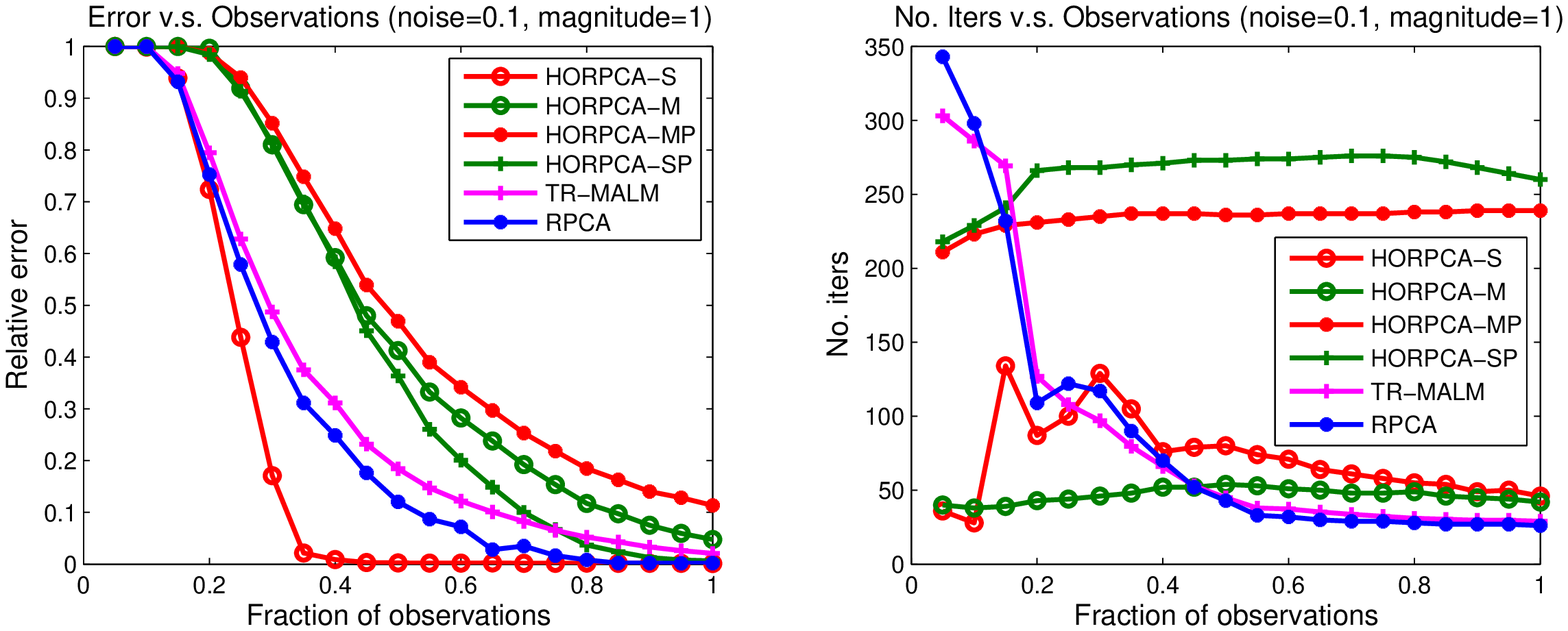}
    \hspace*{0in}\includegraphics[trim = 0.6in 0in 0.5in 0in, clip, width=0.5\textwidth]{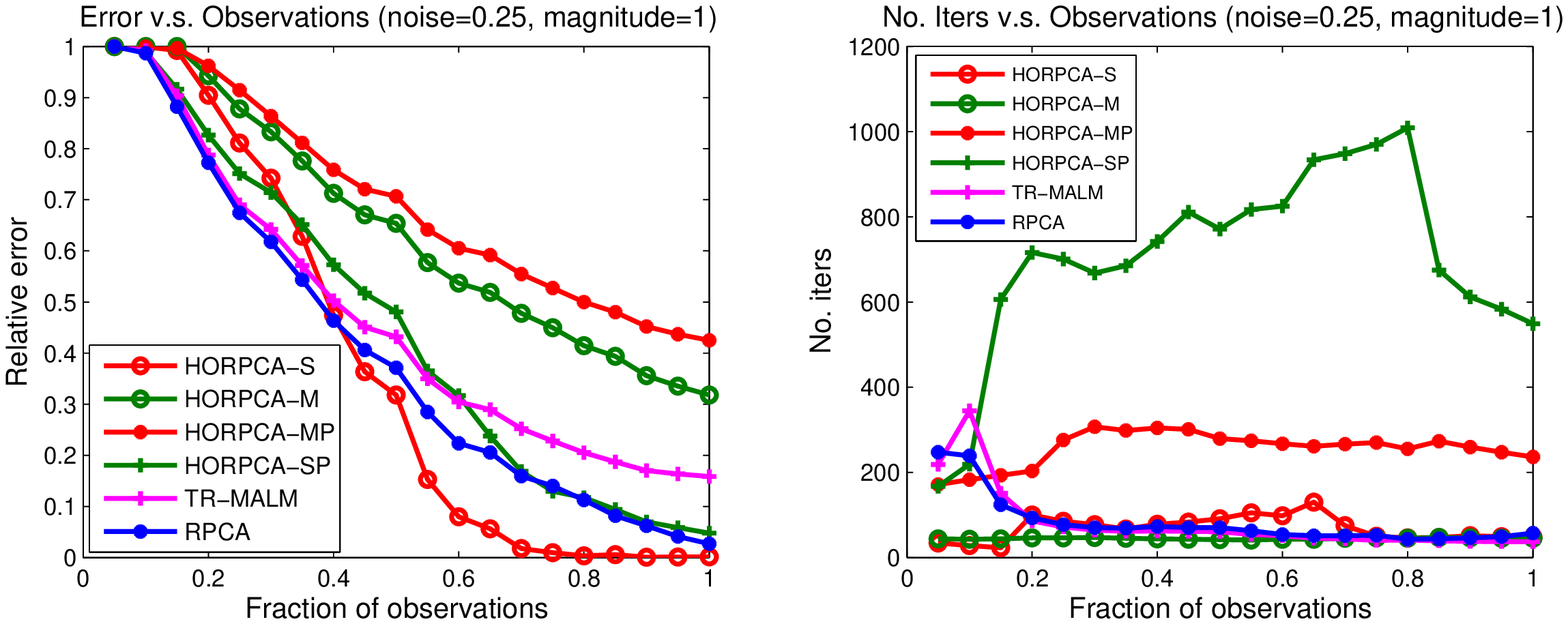}
    \caption{Recovery results for a synthetic rank-(5,5,5) tensor of dimensions (50,50,20) with varying fractions of observations.  $M = 1$.  Left: 10\% data corrupted.  Tucker error = 1.4.  Right: 25\% data corrupted.  Tucker error = 1.9}
    \label{fig:syn-obs-n01}
\end{center}
\end{figure}


\begin{figure}
\begin{center}
    \hspace*{-.1in}\includegraphics[trim = 0.6in 0in 0.5in 0in, clip, width=0.5\textwidth]{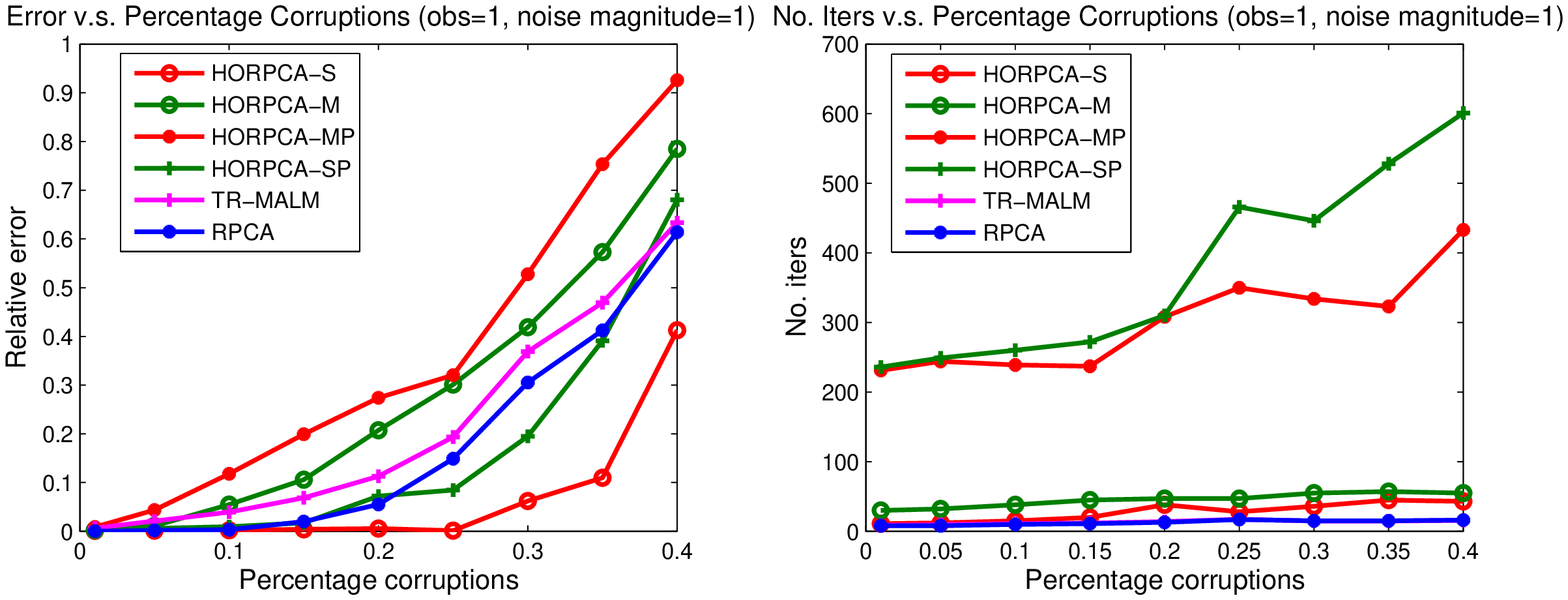}
    \hspace*{0in}\includegraphics[trim = 0.6in 0in 0.5in 0in, clip, width=0.5\textwidth]{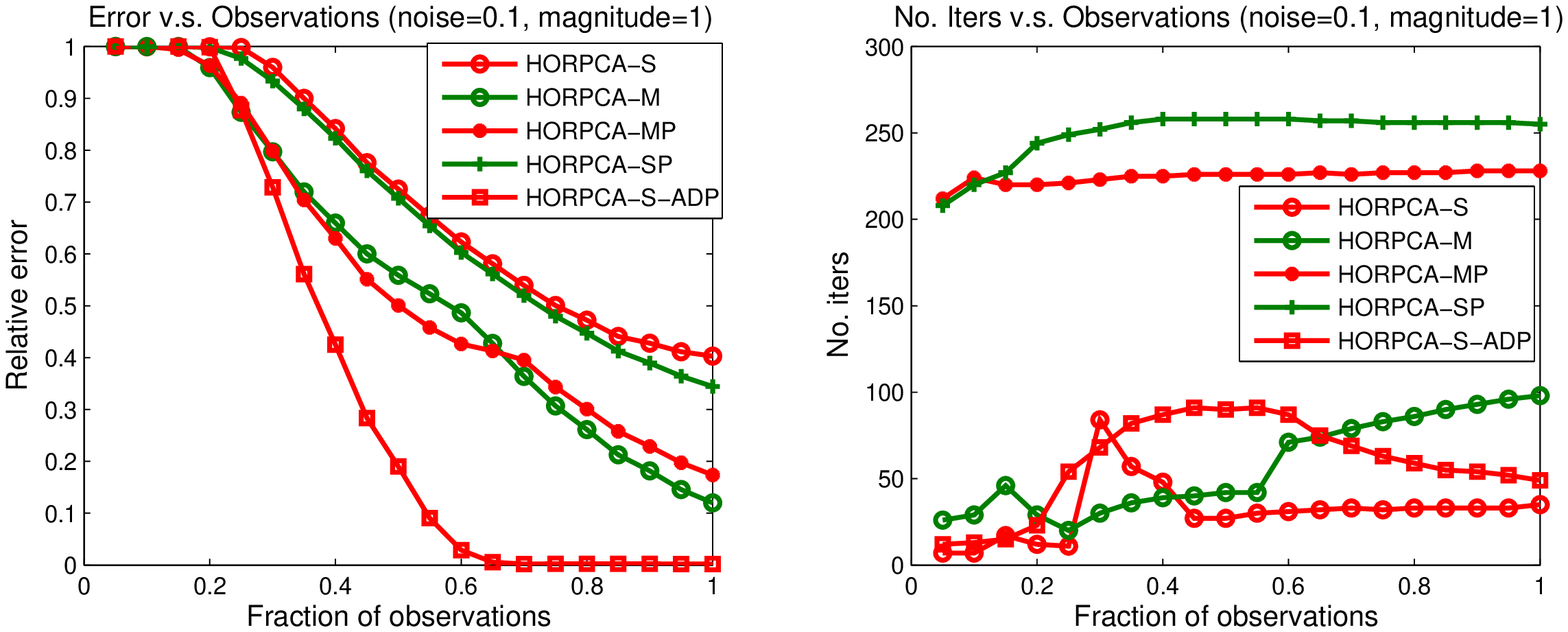}
    \caption{Left: Recovery results for a synthetic rank-(5,5,5) tensor of dimensions (50,50,20) with varying percentage of corruptions.  Full observations were given. $M = 1$.  Right: Recovery results for a synthetic rank-(3,3,20,10) tensor of dimensions (20,20,20,10) with varying fractions of observations.  10\% data corrupted.  $M = 1$.  Tucker error = 0.45.}
    \label{fig:syn-noise-obs100}
\end{center}
\end{figure}

\begin{figure}
\begin{center}
    \hspace*{-0.15in}\includegraphics[width=0.33\textwidth]{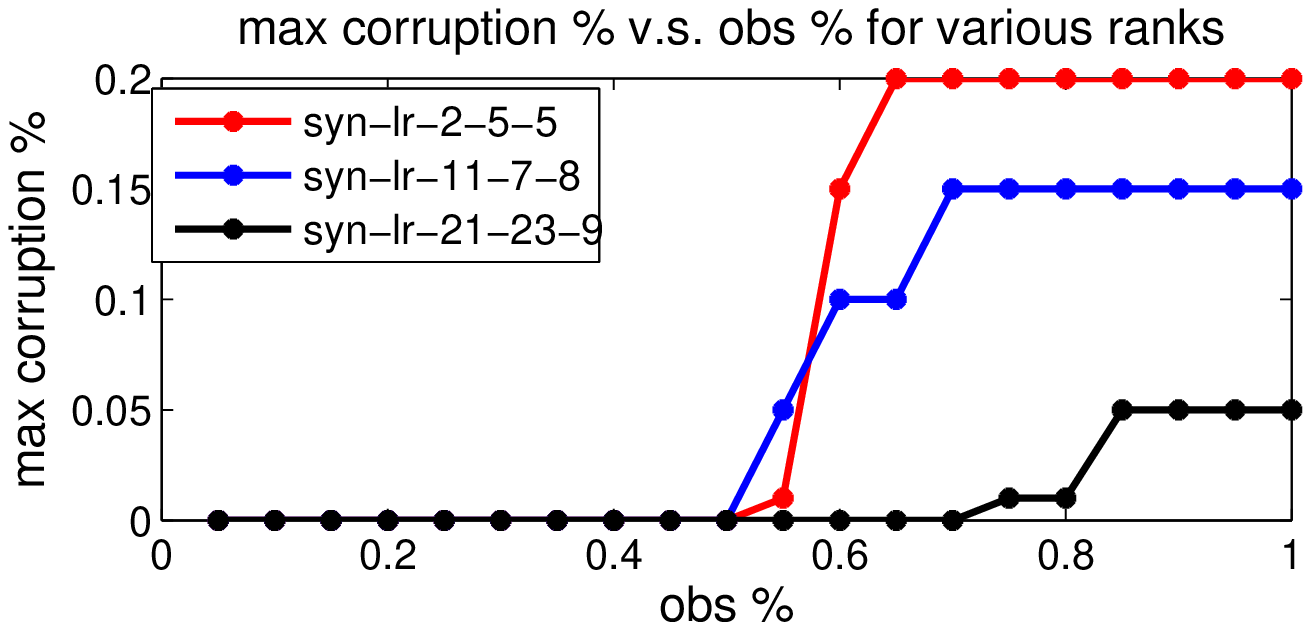}
    \hspace*{0in}\includegraphics[width=0.33\textwidth]{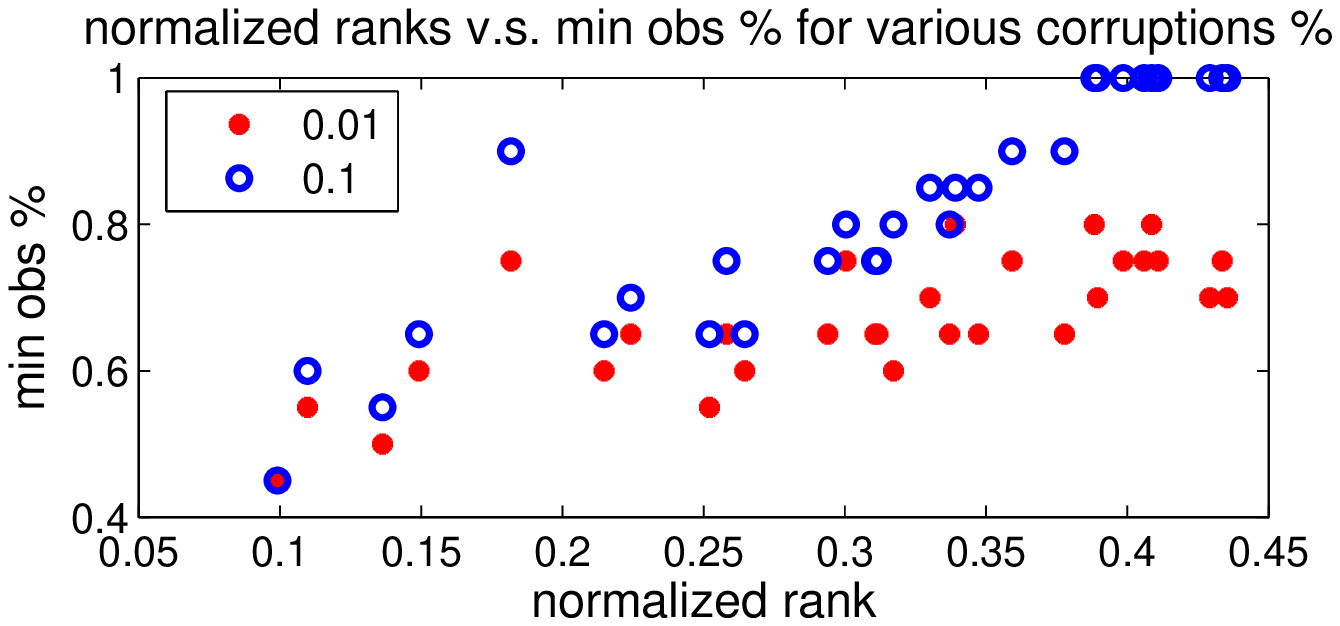}
    \hspace*{0in}\includegraphics[width=0.33\textwidth]{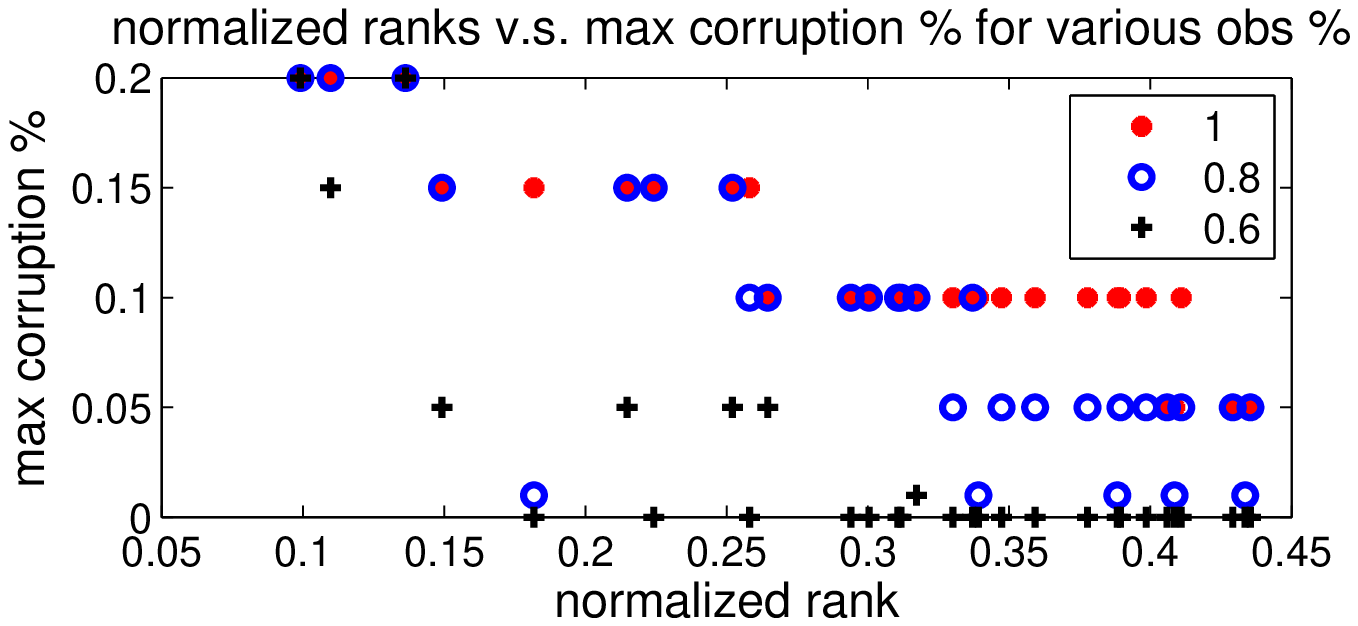}
    \caption{HoRPCA-S recoverability with respect to fractions of observations, sparsity of corruptions, and Tucker rank sums}
    \label{fig:syn-rank-spa-obs}
\end{center}
\end{figure}


\begin{figure}
\begin{center}
    \hspace*{-0.1in}\includegraphics[trim = 0.6in 0in 0.5in 0in, clip, width=0.5\textwidth]{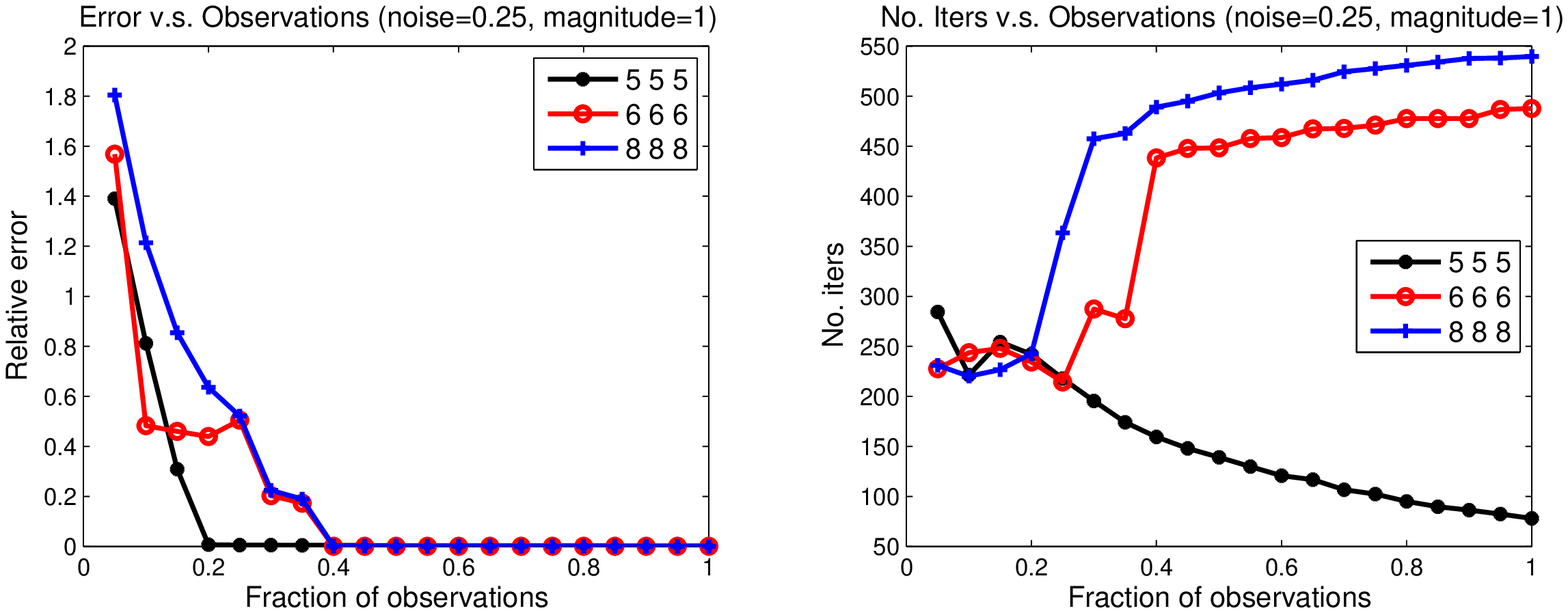}
    \hspace*{0in}\includegraphics[trim = 0.6in 0in 0.5in 0in, clip, width=0.5\textwidth]{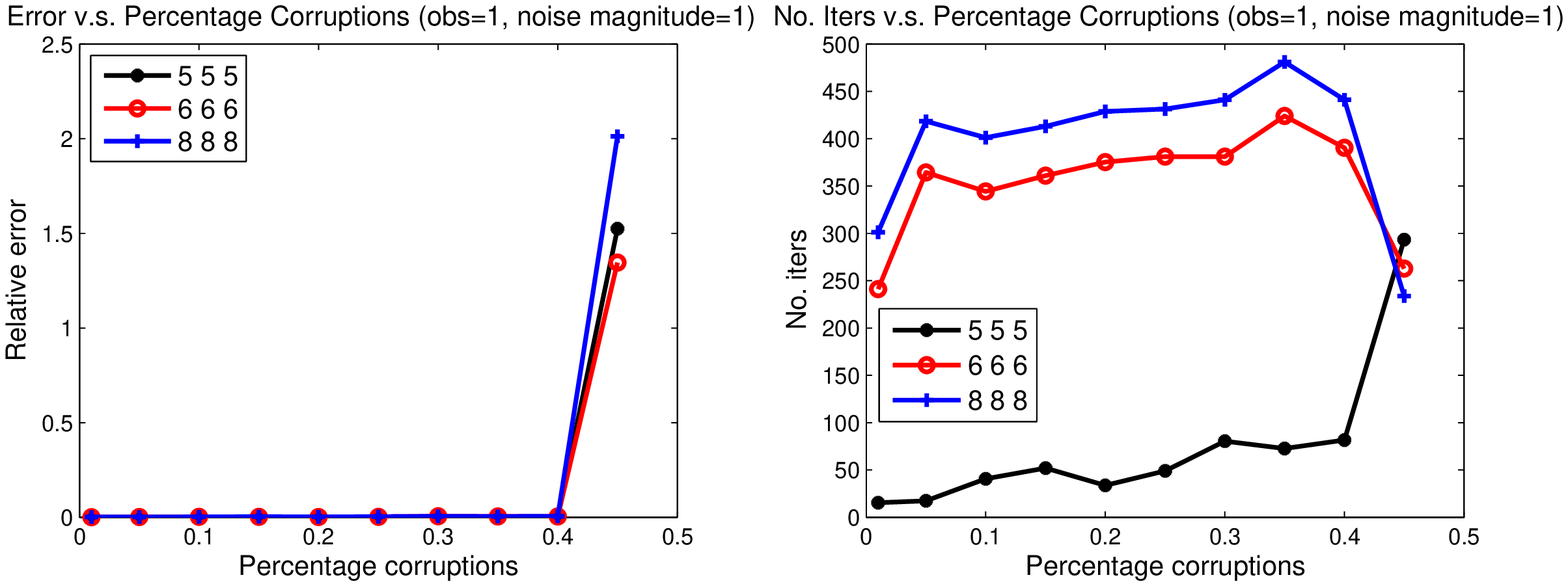}
    \caption{Recovery results for a synthetic rank-(5,5,5) tensor of dimensions (50,50,20), using HoRPCA-C.  $M = 1$.  Left: 25\% data corrupted, with varying percentage of observations.  Right: full observations, with varying percentage of data corrupted.}
    \label{fig:syn-ncx-lr}
\end{center}
\end{figure}

\begin{figure}
\begin{center}
    \includegraphics[trim = 0.6in 0in 0.5in 0in, clip, width=0.5\textwidth]{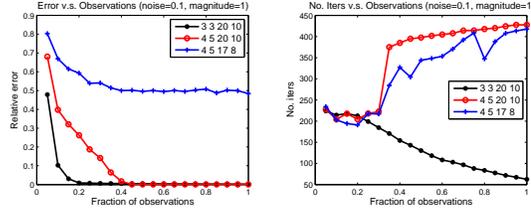}
    \caption{Recovery results for a synthetic rank-(3,3,20,10) tensor of dimensions (20,20,20,10) with varying fractions of observations, using HoRPCA-C.  $M = 1$.  10\% data corrupted.}
    \label{fig:syn-ncx-plr}
\end{center}
\end{figure}

\subsection{3D MRI data}\label{sec:mri3d-dn}
We use a 3D MRI data set INCISIX obtained from the OsiriX repository\footnote{DICOM sample image sets repository. http://www.osirix-viewer.com, http://pubimage.hcuge.ch:8080}, which is also analyzed in \cite{gandy2011tensor}.  We extracted a section of this data set containing 50 slices across a human brain, each having a size of 128 $\times$ 128.  We randomly corrupted 10\% of the data with $M$ = 0.5.  To serve as a baseline, we also applied RPCA to each frontal slice of the data set, and the number of iterations were averaged.  The values of $\lambda_1$ (and $\lambda_*$ in the case of the Lagrangian versions) were tuned for the best perceptual results, i.e. a good balance noise reduction (few bright spots in the black background) and detail retention.  Although the relative error provides a general idea about the quality of the recovered results, we did not use it for tuning the parameters because we found in our experiments that for a given method, the solutions with the lowest relative error tend to appear too noisy to be acceptable.  The value of $\mu$ was set to 40$\times$std(vec($\sbB_\Omega$)), except for RPCA where $\mu =$ 5$\times$std(vec($\sbB_\Omega$)).  We present the recovery results in Figures \ref{fig:incisix-horpca-05} and \ref{fig:incisix-horpca-tc-05}.


The Mixture models (HoRPCA-M and HoRPCA-MP) performed considerably better than the Singleton models and the existing methods in recovering the details of the MRI images from gross corruptions.  This can probably be explained by the fact that this MRI tensor is not truly low-rank in all modes,
reinforcing what we observed in Section \ref{sec:syn-data}.  The nonconvex model, using the Tucker rank information revealed by the Mixture model, tended to retain more details than the other models, but it also included more noise in the partial data case.

The results obtained by the Lagrangian versions are in general inferior to those produced by the exact (constrained) versions due to the relaxed consistency term.   The impact is especially apparent in the case of partial observations, where artifacts due to consistency errors are visible in the recovered images.  In terms of running time, the FISTA-based algorithms generally required more iterations to converge than the ADAL-based algorithms.


\begin{figure}
\begin{center}
    \hspace*{0in}\includegraphics[trim = 1in 0.5in 0in 0.3in, clip,width=0.8\textwidth]{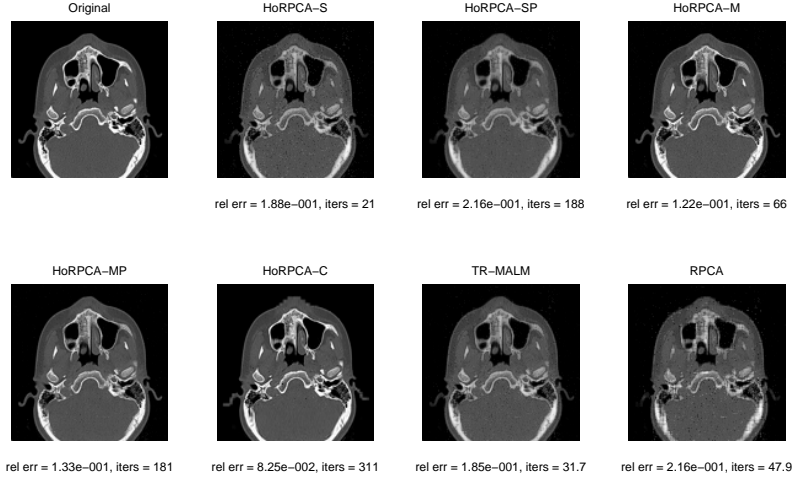}
    \caption{Recovery results for slice 30 of the INCISIX data with full observations.  $M$ = 0.5.}
    \label{fig:incisix-horpca-05}
\end{center}
\end{figure}

\begin{figure}
\begin{center}
    \hspace*{-0.4in}\includegraphics[trim = 1in 0.1in 0in 0.2in, clip,width=0.95\textwidth]{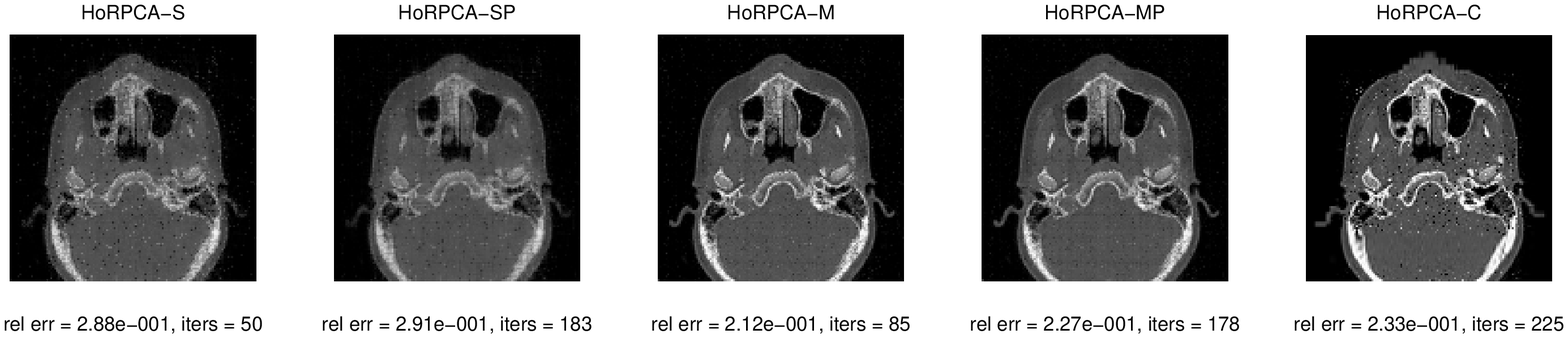}
    \caption{Recovery results for slice 30 of the INCISIX data with 50\% observations.  $M$ = 0.5.}
    \label{fig:incisix-horpca-tc-05}
\end{center}
\end{figure}

%

\subsection{Amino Acid Fluorescence Data}\label{sec:amino}
The amino acid data set is described in Appendix \ref{sec:app_amino_acid}.  It was chosen because it is free of outliers, and hence, the low-rank ground truth is available for reference.  There are two types of outliers defined in the chemometrics literature.  The first type of outliers are observations that have a deviating profile compared to the majority of the observations and are called \textit{outlying samples}.  The second type is called \textit{scatter}, which can be considered as individual outlying data elements (e.g. due to corruption).  Other data elements of the same sample may be fine in this case.  In our experiment, we set maximum magnitude of corruption $M$ = 100, and the fraction corrupted to be 10\%.  This is equivalent to adding scatters.  We did not create outlying samples since the number of original samples was small.

We report the relative errors for HoRPCA variants (see Table \ref{tab:amino}) and plot CP factors of the low-rank tensors obtained by the two algorithms that yielded the smallest relative errors.  For RPCA, we applied the algorithm to each of the three unfoldings of the tensor and report the results corresponding to the mode-3 unfolding, which yielded the lowest relative errors among all three unfoldings.  As described in \cite{tomioka2010estimation}, the CP factors of the Singleton model can be obtained by applying CP decomposition to the core tensor, whereas for the other models, CP decomposition has to be applied directly to the solution tensor.  We also applied CP decomposition with rank 3 directly to the original noiseless tensor and plotted its factors to serve a reference for comparison.

The Tucker rank of the core was obtained by keeping the singular values of the mode-$i$ unfolding of the solution $\sbX$ that are larger than a threshold factor of the largest singular value, for $i = 1,\cdots,N$.  When the threshold was 0.01, the Tucker-ranks of the solutions obtained by HoRPCA-S and HoRPCA-SP were (5,4,3) and (5,5,4).  At a threshold of 0.05, both algorithms identified the correct Tucker rank, (3,3,3).  Even with 50\% data missing, HoRPCA-S was still able to recover the true Tucker rank.  The CP loading factors obtained are plotted in Figures \ref{fig:amino-cp-plots} and \ref{fig:amino-cp-tc-plots}, which show the faithful reconstruction of the CP factors.  The results presented here also shows that the Singleton model is better than the Mixture model at recovering tensors that are low-rank in all modes.  The nonconvex model HoRPCA-C with supplied Tucker rank (3,3,3) also performed well, especially when only partial observations were available.  The CP decomposition applied to the corrupted tensor directly with the correct rank supplied yielded an error not much larger than HoRPCA-S and HoRPCA-C, benefiting from the fact that the maximum magnitude of the errors was not large compared to the magnitude of the uncorrupted entries.

\begin{table}
\begin{center}
\small{
\begin{tabular}{|c|c|c|c|c|}
  \hline
  Fraction data & 100\% &  & 50\% &  \\
  \hline
  Algorithms & Rel. err. & Iters & Rel. err. & Iters \\
  \hline
  HoRPCA-S & 0.0100 & 92 & 0.0332 & 115 \\
  HoRPCA-SP & 0.0098 & 256 & 0.2006 & 255 \\
  HoRPCA-M & 0.0933 & 206 & 0.3097 & 259 \\
  HoRPCA-MP & 0.0896 & 224 & 0.2891 & 219 \\
  HoRPCA-C & 0.0262 & 121 & 0.0268 & 105 \\
  TR-MALM & 0.0472 & 103.7 & 0.2096 & 52.3 \\
  RPCA-3 & 0.0111 & 54 & 0.0596 & 99\\
  PARAFAC & 0.0292 & - & - & - \\
  \hline
\end{tabular}
}
\caption{Reconstruction results for the amino acids data.}
\label{tab:amino}
\end{center}
\end{table}

\begin{figure}
\begin{center}
    \hspace*{0in}\includegraphics[trim = 0.3in 0.5in 0.6in 0.1in, clip,width=0.9\textwidth]{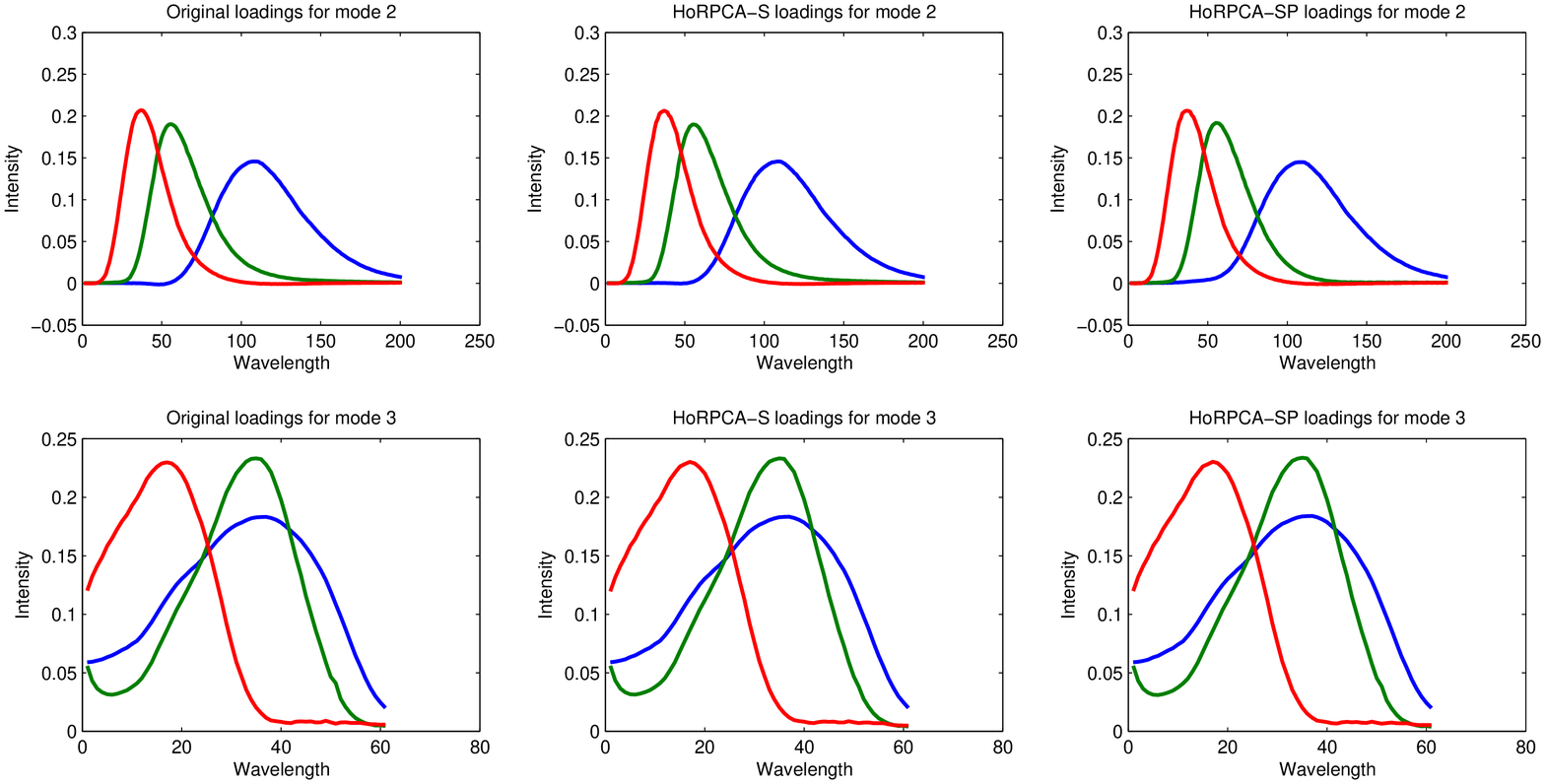}
    \caption{Plots of the original and reconstructed CP factors for the amino acids data.  $M$ = 100.}
    \label{fig:amino-cp-plots}
\end{center}
\end{figure}

\begin{figure}
\begin{center}
    \hspace*{0in}\includegraphics[trim = 0.3in 0.5in 0.6in 0.1in, clip,width=0.7\textwidth]{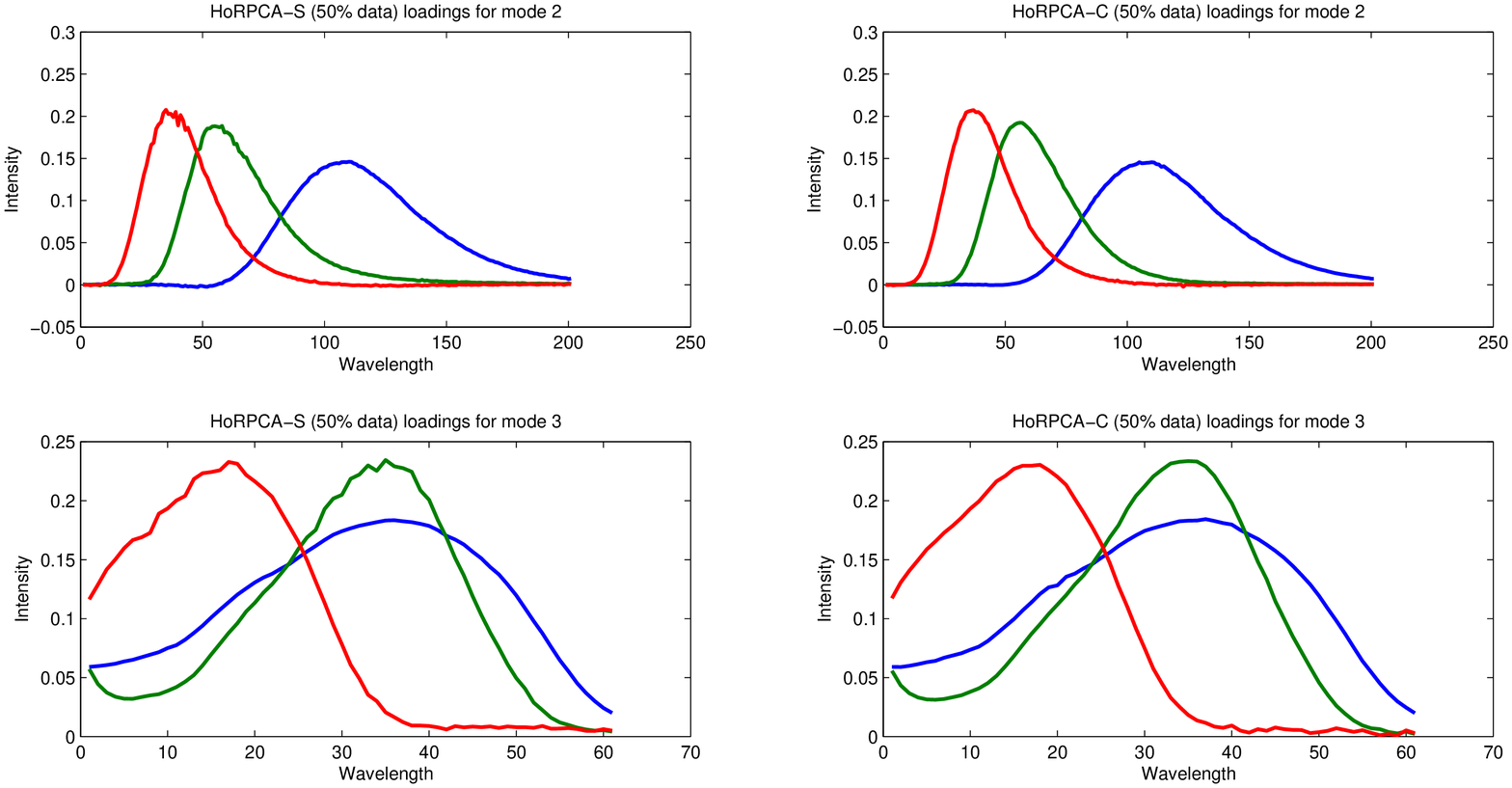}
    \caption{Plots of the reconstructed CP factors for the amino acids data.  $M$ = 100. 50\% observations.}
    \label{fig:amino-cp-tc-plots}
\end{center}
\end{figure}

\subsection{Four-component Fluorescence Data with Scattering and Outliers (The Dorrit Data)}\label{sec:dorrit}
The description of this fluorescence EEM data set is in Appendix \ref{sec:app_dorrit_data}.  Since the Singleton and the nonconvex models of HoRPCA showed better recovery results than the Mixture model for low-rank fluorescence data in the previous example, we focus on these two models for the Dorrit data.  We first tested the Singleton model (HoRPCA-S and HoRPCA-SP) and HoRPCA-C in the presence of scattering only.  Four samples which prior research had identified as outliers \cite{riu2003jack} were excluded from the data set.  We applied the three algorithms to the resulting data set and computed the CP loadings by the same procedures as in the previous example.  Figures \ref{fig:dorrit-cp-plots} and \ref{fig:dorrit-cp-tc-plots} shows the CP loadings corresponding to the emission and the excitation modes for the original noisy data and the recovered CP loadings obtained by our proposed algorithms.  We considered both full and partial observations cases.  Comparing with the pure component loadings plot (see, e.g., \cite{baunsgaard1999phd,engelen2009fully}), we can see that the recovered low-rank tensors obtained by our algorithms achieved a high degree of faithfulness, whereas the results of the classical CP were severely corrupted by scattering.  Our recovery results are also comparable to those obtained by the state-of-the-art robust techniques for removing scatter in fluorescence EEM data \cite{engelen2007automatically}.

We also tested the robustness of the algorithms to both scattering and outliers in the data by using all 27 samples.  The Singleton and the nonconvex HoRPCA models were able to remove the adverse effects of the two types of noise to a tangible extent, but the results were not as compelling as in the previous case.  Due to the small number of samples, the presence of the four outliers considerably changed the global properties of the data, which had a negative impact on the performance of HoRPCA.  While the $R_1$-based robust tensor factorization method \cite{huang2008robust} should be able to eliminate the outliers, it does not take into account the low-rank structure in the first mode. Hence, there was no tangible improvement in the recovery results from this method.

\begin{figure}
\begin{center}
    \hspace*{-0in}\includegraphics[trim = 0.6in 0in 0.5in 0in, clip,width=1\textwidth]{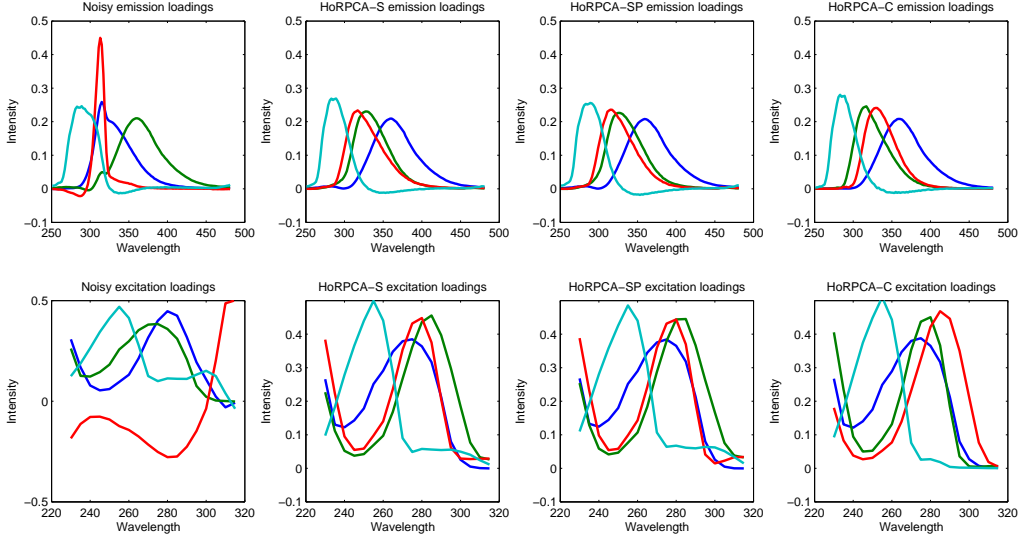}
    \caption{Plots of the original and reconstructed CP factors for the Dorrit data with scattering only.}
    \label{fig:dorrit-cp-plots}
\end{center}
\end{figure}

\begin{figure}
\begin{center}
    \hspace*{0in}\includegraphics[trim = 0.7in 0in 0.4in 0in, clip,width=0.9\textwidth]{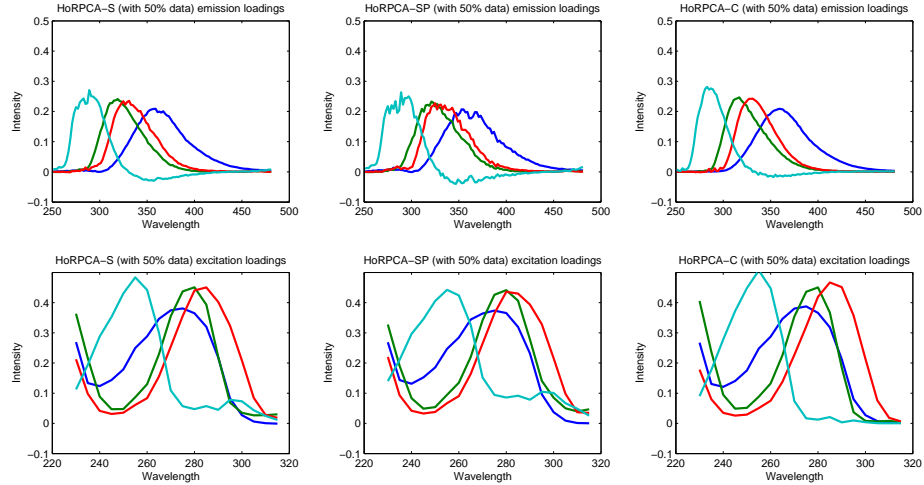}
    \caption{Plots of the reconstructed CP factors for the Dorrit data with scattering only.  50\% observations.}
    \label{fig:dorrit-cp-tc-plots}
\end{center}
\end{figure}


\subsection{Face Representation}\label{sec:YaleB}
We describe the YaleB database in Appendix \ref{sec:app_YaleB}.  In \cite{vasilescu2003multilinear}, it is shown that strategic dimensionality reduction in illuminations can be achieved through the truncated Tucker decomposition, and the resulting compression quality is significantly better than using PCA.  Here, we demonstrate that by recovering a low-Tucker rank tensor, we can essentially reconstruct the low-rank Tucker decomposition where illuminations are automatically reduced.

We corrupted 10\% of the pixels with $M = $1.   As in Section \ref{sec:mri3d-dn}, we also used the ``perceptual error'' to tune the parameters and measure recovery performance.  Good perceptual results in this case means low noise level, reduced shadow, and good retention of other details.  Figure \ref{fig:YaleB-20} presents the reconstruction results.  The number of iterations required by the algorithms are reported below each image.  We can see from Figure \ref{fig:YaleB-20} that the constrained nonconvex model rendered the best performance in terms of noise removal and shadow reduction, followed by the Singleton model.  (See the reduction in the shadow of the nose, for example.)  However, HoRPCA-S was much faster than HoRPCA-C, echoing its stronger convergence guarantee.  RPCA was also quite effective in reducing the shadow, but the area where the shadow used to be and the left eye appear noisier.


\begin{figure}
\begin{center}
    \hspace*{0in}\includegraphics[trim = 1in 0.2in 1in 0.2in, clip,width=0.8\textwidth]{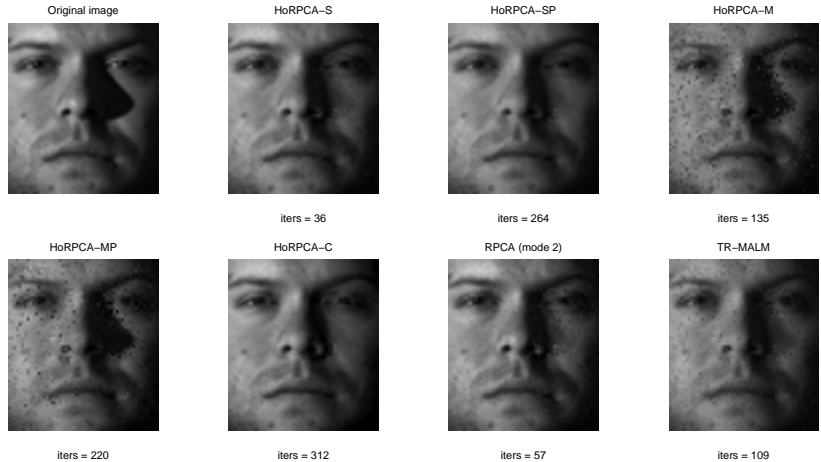}
    \caption{Reconstruction results for the YaleB face image ensemble with full observations.  $M$ = 1.  The images are from illumination 20 of person 1.}
    \label{fig:YaleB-20}
\end{center}
\end{figure}

\subsection{Low-rank Static Background Reconstruction (The Game Data)}\label{sec:game}
In this experiment, we investigated the capability of HoRPCA (HoRPCA-S and HoRPCA-C in particular) to remove the dynamic foreground objects and reconstruct the game background with only a small fraction of pixels available.  The data set is described in Appendix \ref{sec:app_game_data}. As a baseline for comparison, we also implemented an ADAL algorithm for RPCA with missing data estimation and applied it to the mode-4 unfolding of the data tensor.  Note that recovering the foreground objects is not the goal of this experiment because as discussed in Section \ref{sec:missing-data}, recovering the sparse term exactly is not possible with partial observations.

In Figure \ref{fig:game-tc-03-horpca}, we compare the recovery performance of HoRPCA-S, HoRPCA-C, and RPCA with 20\% pixels available.  For each of the algorithms, we tuned the parameter $\lambda_1$ for the best perceptual results.  For HoRPCA-S, we used the adaptive weights for the nuclear norms as in Remark \ref{rem:adaptive_single}.  All three methods were able to separate the roaches and the sandal from the table.  HoRPCA-S performed visibly better than RPCA in estimating the missing pixels and reconstruct the low-rank background.  In particular, the pixels for the table were estimated almost perfectly.  The good reconstruction performance of HoRPCA-S is due to the fact that the background of the frames is low-rank not only in the temporal space but also in the spacial sense.  Hence, even with a very small fraction of pixels, HoRPCA-S was able to filter out the foreground objects and reconstruct the background to an high accuracy.  Based on the Tucker rank information revealed by HoRPCA-S, the nonconvex model HoRPCA-C was able to obtain equally good recovery result, with even more background details retained.


\begin{figure}
\begin{center}
    \hspace*{0in}\includegraphics[trim = 4in 2in 2.5in 0.2in, clip,width=0.37\textwidth]{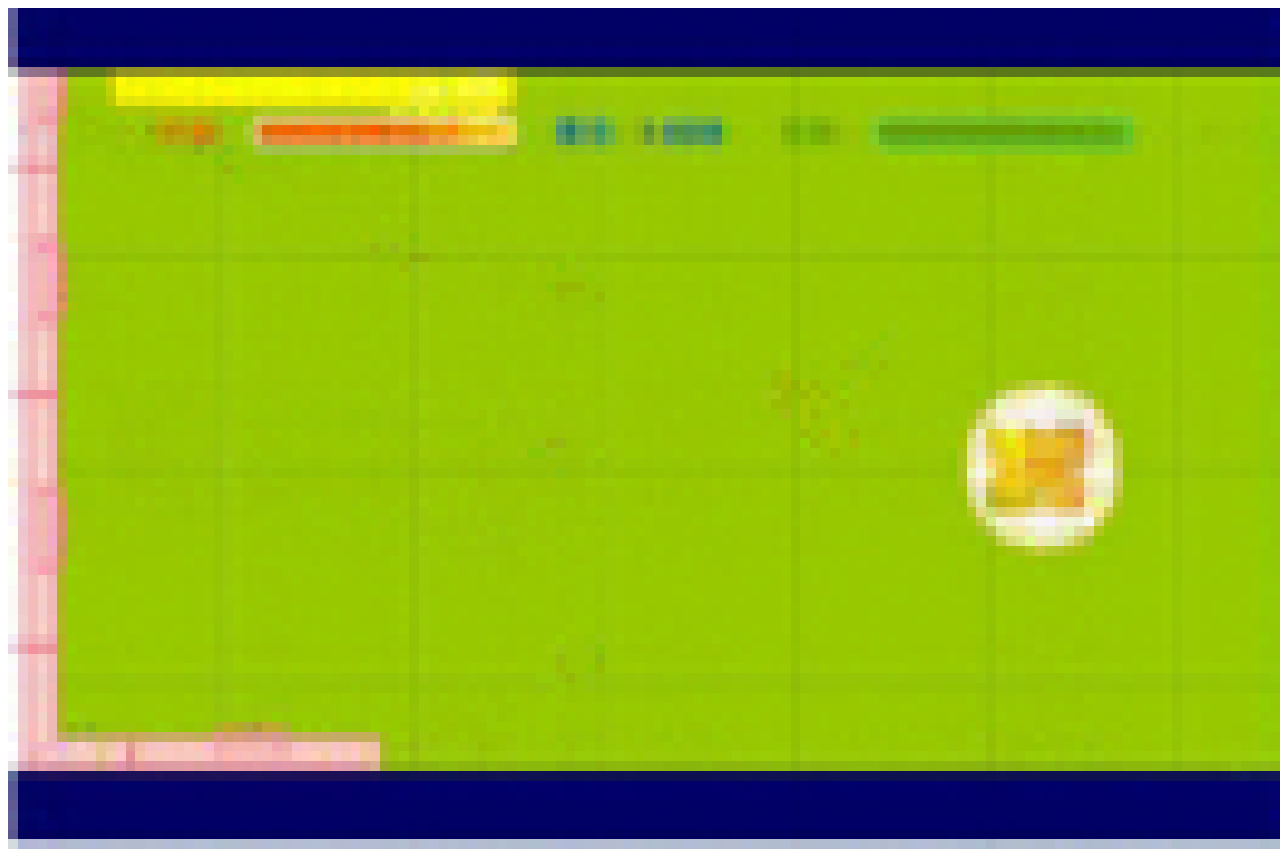}
    \hspace*{-0.5in}\includegraphics[trim = 4in 2in 2.5in 0.2in, clip,width=0.37\textwidth]{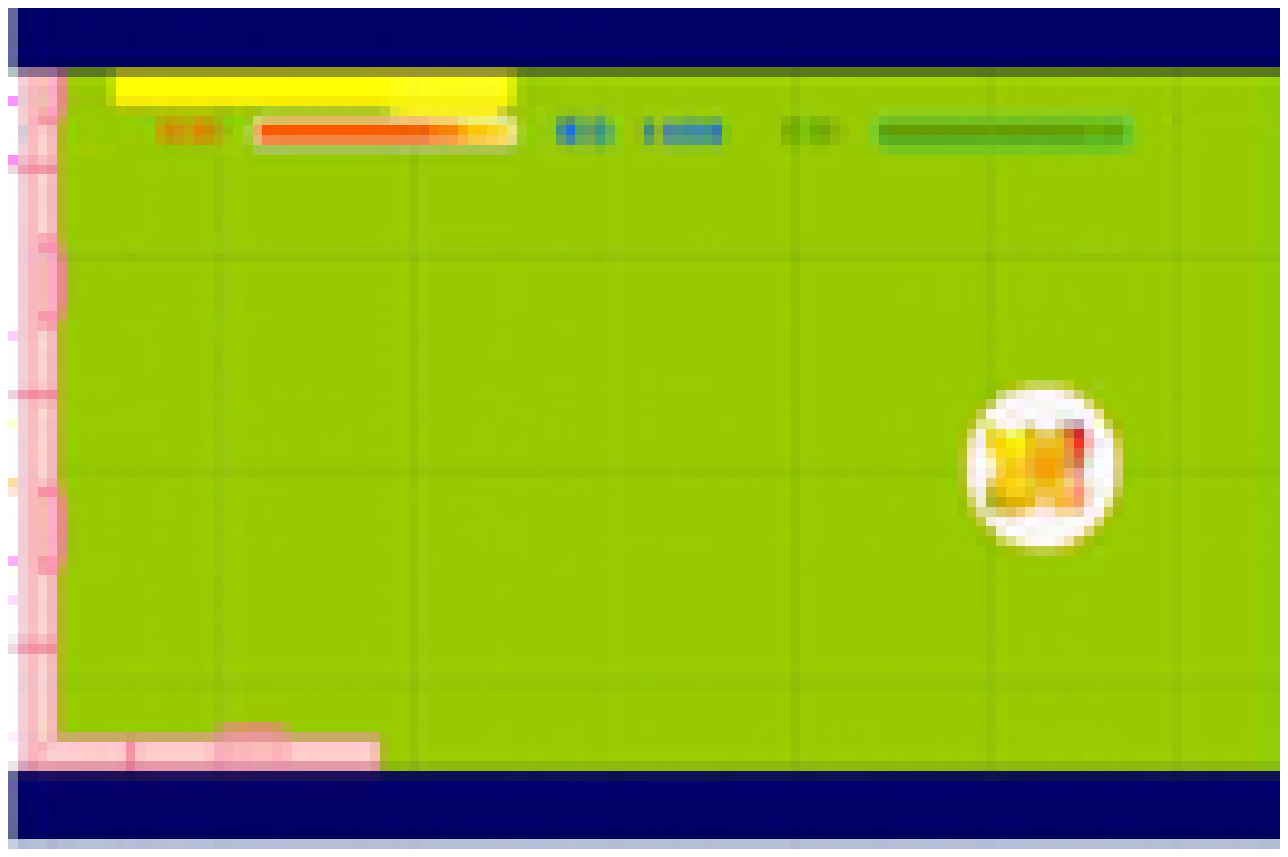}
    \hspace*{-0.5in}\includegraphics[trim = 4in 2in 2.5in 0.2in, clip,width=0.37\textwidth]{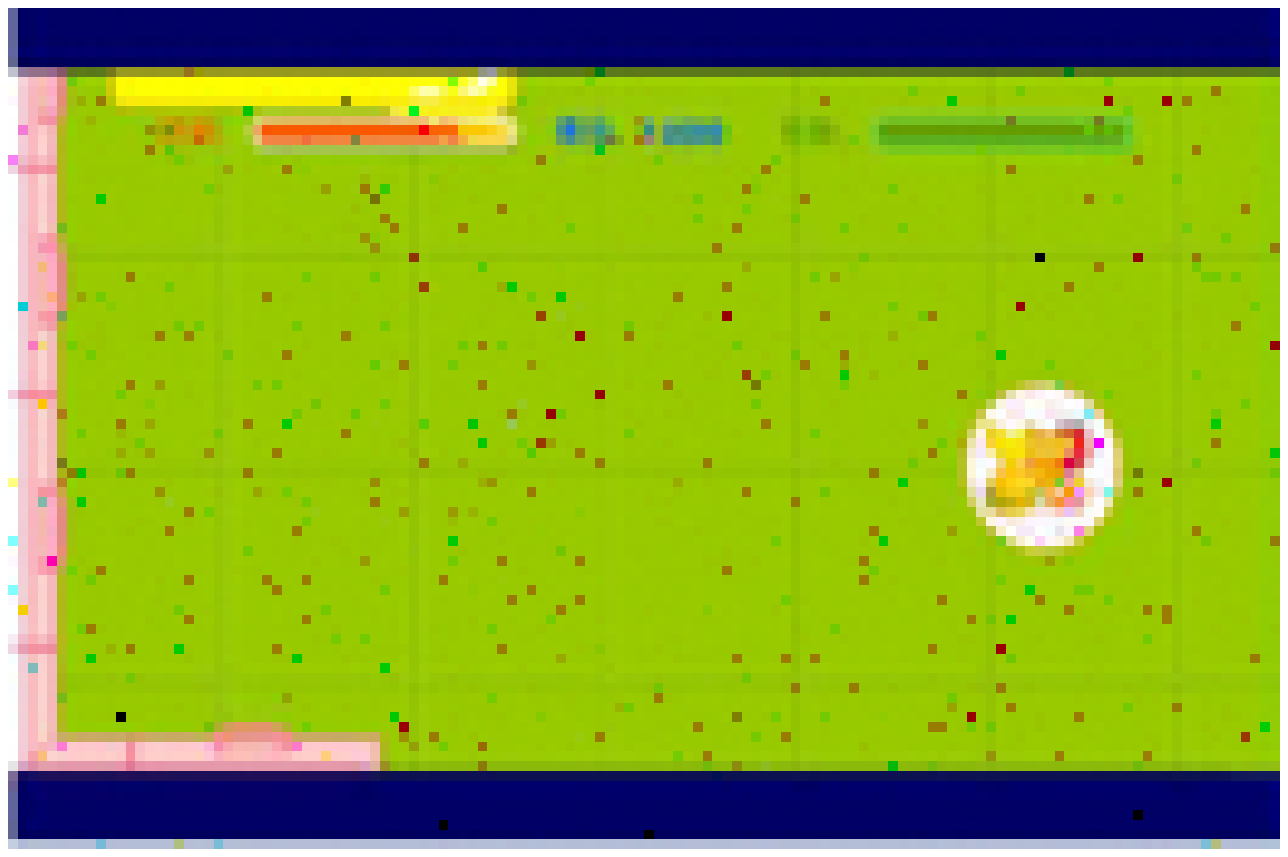}
    \caption{The reconstruction results for the 10th frame obtained by HoRPCA-S (left), HoRPCA-C (center) and RPCA (right) using 20\% noiseless data.}
    \label{fig:game-tc-03-horpca}
\end{center}
\end{figure}


\section{Conclusion}
Robust low-rank tensor recovery plays an instrumental role in robustifying tensor decompositions, and it is also useful in its own right.  The problem of recovering a low-rank tensor from sparse gross corruptions and missing values has not been previously studied in depth either theoretically or computationally.  In this paper, we have focused on the computational aspect of this problem and presented two models in a convex optimization framework HoRPCA, one of which naturally leads to a robust version of the Tucker decomposition.  Both the constrained and the Lagrangian formulations of the problem were considered, and we proposed efficient optimization algorithms with global convergence guarantees for each case.  We also presented a nonconvex model, which can potentially take advantage of more precise rank information.

Our computational experiments cover a rich set of synthetic and real data.  We analyzed the empirical conditions under which exact recovery of a low-rank tensor is possible for the Singleton model of HoRPCA, and we have demonstrated that this model performed the best among the convex models in terms of recovery accuracy when the underlying tensor was low-rank in all modes, whereas the Mixture model performed the best when the tensor was low-rank in only some modes.  When more precise rank information (possibly from the results obtained from the convex models) is available, the nonconvex model exhibited tangibly better recovery performance than the convex relaxations.  From our computational experience, a general strategy is to apply HoRPCA-S first, when no prior rank information is available.  If the revealed ranks indicate that the data may be partially low-rank, HoRPCA-S-ADP or HoRPCA-M should be used instead.  Finally, HoRPCA-C can be used as a refinement step based on the more precise rank information revealed.

Future research directions include studying conditions for exact recoverability in terms of the fraction of observations, the sparsity of the errors, and the Tucker rank; automatic learning of the penalty weights for the adaptive version of the HoRPCA Singleton model \eqref{eq:horpca-single-adapt}; and applying the methods to applications where matrix-based approaches are insufficient.

\section*{Acknowledgments}
This research was supported in part by NSF Grant DMS-1016571, ONR Grant N00014-08-1-1118, and DOE Grant DE-FG02-08ER25856.

\bibliographystyle{abbrv}
\bibliography{tony_bib}

\newpage
\appendix
\section{Convergence Analysis of Algorithm \ref{alg:adal-horpca-mix}}\label{sec:app_conv_iadal}
The proof for Theorem \ref{thm:horpca_m_conv} follows closely the one given by Ma et. al. \cite{ma2012admm}, which is a generalization of the proof by Yang and Zhang for an inexact ADAL method for compressed sensing \cite{yang2009alternating}.  Hence, we give only a sketch of the proof and establish the connection with the previous work.
\begin{proof}
First, we introduce some notation and definitions that are required for the proof.  We define $\sbU := \textrm{TArray}(\sbXbar, \sbLambda)$, and $\sbU\supIndk := \textrm{TArray}(\sbXbar\supIndk, \sbLambda\supIndk)$.  The linear operator $\calA: \bbR^{\NIoneTimesToIN}\rightarrow\bbR^{\IoneTimesToIN}$ is the summation operator defined on $\sbXbar$, i.e. $\calA(\sbXbar) := \sumOneToN\sbX_i$.  We recall from Section \ref{sec:notation} that the matrix $\bA \in \bbR^{I_1\times N\cdot I_1}$ corresponding to $\calA$ is $\left(
                                                                                                          \begin{array}{ccc}
                                                                                                            \bI & \cdots & \bI \\
                                                                                                          \end{array}
                                                                                                        \right)
$, and it can be verified that the maximum eigenvalue of $\bA^T\bA$ equals $N$.
The linear operator $\calH: \bbR^{\xTimesToX{I_1}{I_{N+1}}}\rightarrow\bbR^{\xTimesToX{I_1}{I_{N+1}}}$ is defined as $\calH(\sbU) := \textrm{TArrau}(\frac{1}{\eta\mu}\sbXbar, \mu\sbLambda)$.  We define the norm $\|\sbU\|^2_{\calH} := \langle\sbU,\calH(\sbU) \rangle$ and the inner product $\langle\sbU,\sbV \rangle_{\calH} := \langle\sbU,\calH(\sbV)\rangle$, where $\sbV$ is another tensor array with the same size as $\sbU$.  We also define two convex functions $F(\cdot)$ and $G(\cdot)$ as $F(\sbE) := \lambda_1\|\sbE\|_1, \quad\quad G(\sbXbar) := \sumOneToN\|\modeiXi\|_*$.
Now, problem \eqref{eq:horpca-mix}, HoRPCA with the Mixture model can be written as
\begin{eqnarray}\label{eq:horpca-mix-simplified}
  \min_{\sbXbar,\sbE} \left\{ F(\sbE) + G(\sbXbar) \vbar \sbE + \calA(\sbXbar) = \sbB \right\},
\end{eqnarray}
and one iteration of Algorithm \ref{alg:adal-horpca-mix} is
\begin{equation}\label{eq:horpca-mix-one-iter}
    \left\{
      \begin{array}{ll}
        \sbXbar\supIndkpone &:= \hbox{$\arg\min_{\sbXbar}\left\{G(\sbXbar) + \frac{1}{\eta\mu}\|\sbXbar - (\sbXbar\supIndk - \eta\calA^*(\sbE\supIndk + \calA(\sbXbar) - \sbB - \mu\sbLambda\supIndk))\|^2 \right\}$;} \\
        \sbE\supIndkpone &:= \hbox{$\arg\min_{\sbE}\left\{F(\sbE) - \langle \sbLambda\supIndk, \sbE+\calA(\sbXbar\supIndkpone)-\sbB \rangle + \frac{1}{2\mu}\|\sbE+\calA(\sbXbar\supIndkpone)-\sbB\|^2 \right\}$;} \\
        \sbLambda\supIndkpone &:= \hbox{$\sbLambda\supIndk - \frac{1}{\mu}(\sbE\supIndkpone+\sbXbar\supIndkpone-\sbB)$.}
      \end{array}
    \right.
\end{equation}

Now, with the above definitions and notation in place, we can prove that the iterates $\sbU\supIndk$ satisfy
    $\|\sbU\supIndk - \sbU^*\|^2_{\calH} - \|\sbU\supIndkpone - \sbU^*\|^2_{\calH} \geq \alpha\|\sbU\supIndk - \sbU\supIndkpone\|^2_{\calH}$,
where $\alpha$ is a positive scalar, following the proof in \cite{ma2012admm} and \cite{yang2009alternating}.
Then, it follows that the lemma below holds.
\begin{lem}\label{lem:iadal-lem}
Let $(\sbXbar^*,\sbE)$ be the optimal solution to problem \eqref{eq:horpca-mix-simplified}, $\sbLambda^*$ be the optimal Lagrange multiplier associated with the equality constraint, and $\sbU^* := \textrm{TArray}(\sbXbar^*, \sbLambda^*)$.  Assuming that the proximal gradient step size $\eta$ satisfies $0 < \eta < \frac{1}{\lambda_{\textrm{max}}(\bA^T\bA)} = \frac{1}{N}$, we have the following results regarding the sequence $\{\sbU\supIndk\}$ generated by Algorithm \ref{alg:adal-horpca-mix}:
\begin{itemize}
  \item $\|\sbU\supIndk - \sbU\supIndkpone\|_{\calH} \rightarrow 0$;
  \item $\{\sbU\supIndk\}$ lies in a compact region;
  \item $\|\sbU\supIndk - \sbU^*\|^2_{\calH}$ is monotonically non-increasing and thus converges.
\end{itemize}
\end{lem}

Lemma \ref{lem:iadal-lem} implies that the sequence $\{(\sbU\supIndk,\sbE\supIndk)\}$ has a subsequence that converges to $(\hat{\sbU},\hat{\sbE})$.  By considering the optimality conditions of the two subproblems in \eqref{eq:horpca-mix-one-iter}, we can then show that any limit point $(\hat{\sbU},\hat{\sbE})$ of the sequence $\{(\sbU\supIndk,\sbE\supIndk)\}$ satisfies the KKT conditions for problem \eqref{eq:horpca-mix-simplified}.  Hence, any limit point of $\{(\sbXbar\supIndk,\sbE\supIndk)\}$ is an optimal solution to problem \eqref{eq:horpca-mix}.
\end{proof}

\section{Lagrangian Version of HoRPCA via FISTA: global convergence rate}\label{sec:app_fista}
Observe that both problems \eqref{eq:horpca-unc} and \eqref{eq:horpca-mix-unc} are in the generic form
\begin{equation}\label{eq:horpca-mix-unc-generic}
    \min_{\sbXtild} \quad l(\sbXtild) + R(\sbXtild),
\end{equation}
where $R(\sbXtild) := \lambda_*\sumOneToN \|\modeiXi\|_* + \lambda_1\|\sbE\|_1$, and $l(\sbXtild) := \frac{1}{2}\|\calA(\sbXtild) - \tilde{\sbB}\|^2$.    The loss function $l(\cdot)$ is thus Lipschitz continuous, and the Lipschitz constant $L(l)$ associated with the gradient $\nabla l$ is $N+1$ for both models.  Hence, no line-search is required for the FISTA Algorithm \ref{alg:fista-horpca}.  In addition, the regularization function $R(\cdot)$, i.e. the sum of the trace norms and the $L_1$ norm, is decomposable in the decision variables, and the proximal operator associated with each of those norms can be computed in closed-form, i.e. $\calT_\mu(\cdot)$ for the trace norms and $\calS_\mu(\cdot)$ for the $L_1$ norm.

It has been shown that this class of algorithms achieves the optimal global convergence rate (in terms of iteration complexity) of $\calO\left(\frac{1}{\sqrt{\epsilon}}\right)$ among all the first-order methods, where $\epsilon$ is the target distance between the obtained solution and the optimal solution.

Small values of $\lambda_*$ and $\lambda_1$ often lead to slow convergence of FISTA.  To alleviate that problem, we adopt a fast continuation scheme which has been shown to be effective in \cite{lin2009fast,toh2010accelerated}.  We express $\lambda_1$ as $r\lambda_*$, where $r$ is a given parameter.  The key is to start the algorithm with a large initial $\lambda_*^0$ and decrease $\lambda_*$ by a factor of $\eta$ after each iteration until it reaches the desired minimum value $\bar{\lambda_*}$.  In our experiments, we found that the speedup gained from using this scheme as opposed to applying FISTA directly with a small $\lambda_*$ could be more than an order of magnitude.

Since the function $l(\cdot)$ in problem \eqref{eq:horpca-mix-unc-generic} is continuously differentiable with a Lipschitz-continuous gradient, we have the following global convergence rate result for Algorithm \ref{alg:fista-horpca} as a consequence of Theorem 4.4 in \cite{beckteboulle}.  The continuation scheme does not affect the validity of the proof since $\lambda_*$ converges to $\bar{\lambda_*} > 0$ in the long run \cite{lin2009fast}.
\begin{thm}\label{thm:horpca_p_conv}
Define $F(\sbXtild) := l(\sbXtild) + R(\sbXtild)$, the objective function of problem \eqref{eq:horpca-mix-unc-generic}.  Let $\{\sbXtild\supIndk,\sbYtild\supIndk\}$ be the sequence generated by Algorithm \ref{alg:fista-horpca}, and $\bar{k} = \log_\eta\left( \frac{\lambda_*^0}{\bar{\lambda_*}} \right)$.  Then for any integer $k>\bar{k}$, $F(\sbXtild\supIndk) - F(\sbXtild^*) \leq \frac{2\left( \sumOneToN\|\sbX_i^{(k)}-\sbX_i^*\|^2 + \|\sbE^{(k)}-\sbE^*\|^2 \right)}{(k+1-\bar{k})^2}$,
where $\sbXtild^*$ is the optimal solution of problem \eqref{eq:horpca-mix-unc-generic}.
\end{thm}

\section{Proof of Lemma \ref{lem:ncx-conv}}\label{sec:app_conv_ncx}
\begin{proof}
First, we use a technique from \cite{shen2011augmented} to convert condition \eqref{eq:ncx-kkt3} into a more useful form.  Note that \eqref{eq:ncx-kkt3} is equivalent to
\begin{equation}\label{eq:ncx-kkt3a}
   N\sbE + \mu\sumOneToN\sbLambda_i \in \mu\partial\|N\sbE\|_1 + N\sbE \equiv \calQ_\mu(N\sbE).
\end{equation}
We can verify that $\calQ_\mu(\cdot) := \mu\partial \|\cdot\|_1 + \cdot$ is monotone element-wise and $\calQ\inv_\mu(\cdot) \equiv \calS_\mu(\cdot)$.  Applying $\calQ^{-1}_\mu(\cdot)$ to both sides of \eqref{eq:ncx-kkt3a} and considering $\fold_i(\bU_i\bV_i) + \sbE = \sbB, i = 1,\cdots,N$, we obtain
\begin{equation}\label{eq:ncx-kkt3b}
    \sbE = \frac{1}{N}\calS_\mu\left( \sumOneToN\sbB + \mu\sbLambda_i - \fold_i(\bU_i\bV_i) \right).
\end{equation}

Since the sequence $\setkOneToInf{\sbW\supIndk}$ is bounded, there exists a subsequence $\setjOneToInf{\sbW\supInd{k_j}}$ that converges to $\sbW^*$, a limit point of $\setkOneToInf{\sbW\supIndk}$.  Let $\modeiXi^* = \bU_i^*\bV_i^*$, obtained from the SVD of $\modeiXi^*$.  By Line \ref{line:horpca-c-E} in Algorithm \ref{alg:horpca-ncx} and the identity $\modeiXi^* = \bU_i^*\bV_i^*$, it is easy to see that $\sbW^*$ satisfies condition \eqref{eq:ncx-kkt3b}.  By Line \ref{line:horpca-c-Lambda} and $\sbLambda_i\supIndkpone - \sbLambda_i\supIndk \rightarrow \bzero$, we have $\fold_i(\bU_i\supIndk\bV_i\supIndk) + \sbE\supIndk - \sbB \rightarrow \bzero$.  Hence, $\sbW^*$ satisfies condition \eqref{eq:ncx-kkt2}.

Now, we focus on condition \eqref{eq:ncx-kkt1}.  By Line \ref{line:horpca-c-X}, we have
\begin{equation}\label{eq:ncx-limit-svd}
    \bU_i^*\bV_i^* = \calP_{r_i}(\bU_i^*\bV_i^* + \mu\modeiLambdai^*).
\end{equation}
We discuss two cases for $\modeiLambdai^*$.  (i) If $\modeiLambdai^* = \bzero$, then $\sbE^* = \bzero$ by condition \eqref{eq:ncx-kkt3}.  In this case, however, $\bU_i^*$ and $\bV_i^*$ cannot satisfy condition \eqref{eq:ncx-kkt2} unless $(r_1,\cdots,r_N)$ is equal to the Tucker rank of $\sbB$, following \eqref{eq:ncx-limit-svd}.  This contradicts our assumptions.  (ii) $\sbLambda_i^*$ is non-zero.  Then, by \eqref{eq:ncx-limit-svd}, it is not hard to see that there exists a factorization $\modeiLambdai^* = \bU_{\sbLambda_i}^*\bV_{\sbLambda_i}^*$ (e.g, by SVD) such that $(\bU_{\sbLambda_i}^*)\transp\bU_i^* = \bzero$ and $(\bV_i^*)\transp\bV_{\sbLambda_i}^* = \bzero$.  It immediately follows that $\bU_i^*,\bV_i^*,\sbLambda_i^*, i=1,\cdots,N,$ satisfy condition \eqref{eq:ncx-kkt1}.
\end{proof}

\section{Data Sets}
\subsection{Amino Acid Fluorescence Data}\label{sec:app_amino_acid}
The amino acid fluorescence excitation-emission (EEM) data was originally prepared by Bro and Andersson to demonstrate the PARAFAC decomposition.  It was also used in \cite{tomioka2010estimation} for tensor completion experiments.  The data measures the fluorescence of five solutions which have different compositions of three amino acids.  The data tensor is of a size 5$\times$201$\times$61, and the three dimensions correspond to emission wavelength, excitation wavelength, and samples, respectively.  Since the underlying factors are the three amino acids, this data set has an approximately low rank of 3 (or a Tucker rank of (3,3,3)).

\subsection{The (partial) YaleB Database}\label{sec:app_YaleB}
The YaleB database is a face image ensemble consisting of different poses and angles of illumination.  We used the cropped version of the database and selected the face images of five people, each has 40 different illuminations and one pose.  Each image is 64$\times$56 grey-scale and vectorized.  The resulting data tensor is 3584$\times$40$\times$5.  Similar face ensembles have been studied using Tucker decomposition, and the factors extracted are used to construct the so-called `Tensor Faces'.

\subsection{The Dorrit Data}\label{sec:app_dorrit_data}
This fluorescence EEM data set was produced by Dorrit Baunsgaard for fundamental studies on PARAFAC \cite{baunsgaard1999phd}.  Synthetic samples containing mixture of four analytes (hydroquinone, tryptophan, phenylalanine and dopa) at different concentrations were measured in a Perkin-Elmer LS50 B fluorescence spectrometer. The preprocessed data \cite{riu2003jack} consists of 27 fluorescence landscapes corresponding to 27 samples, each contains 116 emission wavelengths from 251nm to 481nm at 2nm intervals and 18 excitation wavelengths from 230nm to 315 nm taken every 5 nm.  Hence, the dimension of the data tensor is $27\times 116\times 18$.  Ideally, a four-way CP decomposition is sufficient to analyze this data.  However, this data set is known to contain both scattering and outlying samples \cite{engelen2009fully}, which can adversely degrade the results of CP.  The Dorrit data has often been used for testing robust tensor decomposition algorithms \cite{engelen2011detecting,engelen2009fully,rousseeuw2006robustness}.

\subsection{The Game Data}\label{sec:app_game_data}
We captured a series of screen-shots of a Flash$^\circledR$ game at a frequency of one frame per second.  The goal of the game is to protect the food at the center of the table from emerging roaches.  The player is to use the mouse (which appears as a sandal) to click (hit) on the moving roaches.  We collected 27 colored frames, each of which has a resolution of 86$\times$130, forming a 86$\times$130$\times$3$\times$27 tensor.  Figure \ref{fig:game-orig} shows two frames from the original data.  The dynamic foreground (the roaches and the sandal) occupies a small portion of the frame, and the static background (the table and the major part of the statistics bars) has a simple structure.  Hence, the background of this data set is a typical low-rank tensor.

\begin{figure}
\begin{center}
    \hspace*{0in}\includegraphics[width=0.2\textwidth]{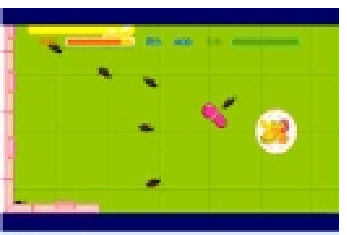}
    \hspace*{0.2in}\includegraphics[width=0.2\textwidth]{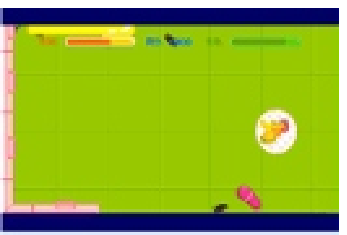}
    \caption{The original 10th and 20th frames of the Game data set.}
    \label{fig:game-orig}
\end{center}
\end{figure}

\section{Parameter Set-up for Algorithms}\label{sec:app_params}
For HoRPCA-S, HoRPCA-M, and RPCA, the parameter to tune is $\lambda_1$ ($\lambda_* = 1)$; for HoRPCA-SP and HoRPCA-MP, we need to tune $\lambda_1$ and $\lambda_*$.  Let $I_\textrm{max} = \max(I_1,\cdots,I_N)$ for both matrices and tensors.  \cite{candes2009robust} suggests that a good choice for $\lambda_1$ is $r \equiv \frac{1}{\sqrt{I_\textrm{max}}}$ for RPCA.  In our experiments, we follow a similar heuristic and set $\lambda_1 = \alpha r\lambda_*$, where $\alpha$ is to be tuned. Our experience is that for the Singleton model, $\alpha$ is usually around 1, and the value of $a$ for the Mixture model is about $\frac{1}{N}$ of that for the Singleton model.  In theory, we need not tune $\lambda_*$ for the Lagrangian version of the models since we let $\lambda_* \rightarrow \bar{\lambda_*}\approx 0$ through continuation.  In practice, we found that when the percentage of corruptions is larger than a threshold (e.g. 20\% for the synthetic data), a very small $\lambda_*$ (and also $\lambda_1$) often produced inferior recovery accuracy.  So in that case, we had to tune $\lambda_*$ as well and kept the minimum value $\bar{\lambda_*}$ higher.  We kept $\mu$ constant for the ADAL-based convex algorithms for simplicity and set $\mu = 10\textrm{std}(\rmvec(\sbB))$ unless otherwise specified.  Note that further tuning of $\mu$ with more sophisticated updating scheme could improve the computational performance of the ADAL algorithms.  For HoRPCA-C, we started with $\mu=1$ and decreased it by a factor of 0.9 every 10 iterations with the minimum value of $1e-4$.  The Tucker-ranks were set to the true values in Sections \ref{sec:amino} and \ref{sec:dorrit} and to the Tucker-ranks obtained from the best-performing convex algorithms for the other sets of real data.  The setting of the Tucker-ranks for the synthetic data was detailed in Section \ref{sec:syn-data}.

The parameters for the continuation scheme of the FISTA-based algorithms were chosen as follows: We set the initial value $\lambda_*^0 = 0.99 \|\sbB\|$ (or $\|\sbB_\Omega\|)$.  The default value of $\bar{\lambda_*}$ was $\delta\lambda_*^0$, where $\delta = $1e-5.  The factor for decreasing $\lambda_*$ in each iteration was $\eta = 0.97$.

We employed the following stopping criteria to determine if an algorithm has converged.  For the ADAL-based algorithms, we monitored the relative primal and dual residuals as in \cite{boyd2010distributed,qin2011structured} and stopped the algorithms when the maximum of the two quantities decreased to below 1e-3. For the FISTA-based algorithms, we stopped the algorithms when the relative primal residual and relative difference between two consecutive iterates were both below 1e-4.  This threshold value was chosen to match the recovery results of the FISTA algorithms roughly with those of the ADAL algorithms, given the same values for $\lambda_1$.


\end{document}